\documentclass{article}

\usepackage{microtype}
\usepackage{graphicx}
\usepackage{subfigure}
\usepackage{booktabs} % for professional tables
\usepackage{hyperref}

\usepackage[accepted]{icml2018}
\usepackage{algorithm}
\usepackage{hyperref}
\usepackage{bm}

\usepackage{natbib}
\usepackage{amsmath}
\usepackage{amsthm}
\usepackage{amssymb}
\usepackage{mathtools}
\usepackage{tikz}
\usepackage{xcolor}
\usetikzlibrary{arrows}

\usepackage{courier}            % standard fixed width font
\usepackage[scaled]{helvet} % see www.ctan.org/get/macros/latex/required/psnfss/psnfss2e.pdf
\usepackage{url}                  % format URLs
\usepackage{listings}          % format code
\usepackage{paralist}
\usepackage{enumitem}      % adjust spacing in enums
\usepackage[normalem]{ulem}
\usepackage{wrapfig}

\newtheorem{theorem}{Theorem}[section]
\newtheorem{definition}[theorem]{Definition}
\newtheorem{lemma}[theorem]{Lemma}

\newcommand{\lemref}[1]{Lemma~\ref{#1}}
\newcommand{\thmref}[1]{Theorem~\ref{#1}}
\newcommand{\A}{\mathcal{A}}
\newcommand{\X}{\mathcal{X}}\newcommand{\Y}{\mathcal{Y}}\newcommand{\Z}{\mathcal{Z}}
\newcommand{\E}{\mathop{\mathbb{E}}}
\newcommand{\N}{\mathbb{N}}\newcommand{\R}{\mathbb{R}}\newcommand{\maps}{\colon}
\newcommand{\F}{\mathcal{F}}

\newcommand{\defref}[1]{Definition~\ref{#1}}

\renewcommand{\P}{\mathop{\mathbb{P}}}

\newcommand{\pa}{{\rm pa}}\newcommand{\nd}{{\rm nd}}

\newcommand{\keyword}[1]{\textbf{#1}}

\newcommand{\approximate}[2]{\alpha_{#2}(#1)}%%{\ensuremath{[ #1 ]_{#2}}}

\newcommand{\mc}[1]{\mathcal{#1}}
\newcommand{\nb}[2]{}%[[#1: \emph{#2}]]}
\newcommand{\set}[1]{\{#1\}}
\newcommand{\tuple}[1]{\langle#1\rangle}
\newcommand{\define}[1]{\begin{definition}#1\end{definition}}
\newcommand{\items}[1]{\begin{compactitem}#1\end{compactitem}}
\newcommand{\mysssection}[1]{\noindent\textbf{#1}\hspace{8pt}}
\newcommand{\hide}[1]{}
\newcommand{\reminder}[1]{{\color{red} *** #1 *** }}

\newcommand{\indict}{\mathbf{1}}
\newcommand{\unif}{\text{ Unif}}

\icmltitlerunning{Discrete-Continuous Mixtures in Probabilistic Programming}
\begin{document}



\twocolumn[
%\icmltitle{The Extended Semantics for Probabilistic Programming Languages\\ with Discrete-Continuous Mixtures} 
\icmltitle{Discrete-Continuous Mixtures in Probabilistic Programming:\\
Generalized Semantics and Inference Algorithms} 



\icmlsetsymbol{equal}{*}

\begin{icmlauthorlist}
	\icmlauthor{Yi Wu}{ucb}
	\icmlauthor{Siddharth Srivastava}{asu}
	\icmlauthor{Nicholas Hay}{vic}
	\icmlauthor{Simon S. Du}{cmu}
	\icmlauthor{Stuart Russell}{ucb}
\end{icmlauthorlist}

\icmlaffiliation{ucb}{University of California, Berkeley}
\icmlaffiliation{asu}{Arizona State University}
\icmlaffiliation{vic}{Vicarious Inc.}
\icmlaffiliation{cmu}{Carnegie Mellon University}

\icmlcorrespondingauthor{Yi Wu}{jxwuyi@gmail.com}

\icmlkeywords{Probabilistic Programming Language, Discrete-Continuous Mixtures, Importance Sampling}

\vskip 0.3in ]
\printAffiliationsAndNotice{}

\begin{abstract}
Despite the recent successes of probabilistic programming languages (PPLs) in AI applications, PPLs offer only limited support for random variables whose distributions combine discrete and continuous elements.
%theorems offer only limited support
%for continuous variables. T
%heorems concerning semantics and
%algorithmic correctness are often limited to discrete variables or, in some cases, bounded continuous variables. 
%We show natural examples that violate these restrictions and break standard algorithms for PPLs. 
%Using {\em probability kernels}, the measure-theoretic generalization of conditional distributions, 
We develop the notion of
{\em measure-theoretic Bayesian networks (MTBNs)} and use it to provide more general semantics for PPLs with arbitrarily many random variables defined over arbitrary measure spaces. 
We develop two new general sampling algorithms that are provably correct under the MTBN framework: the lexicographic likelihood weighting (LLW) for general MTBNs and the lexicographic particle filter (LPF),  a specialized algorithm for state-space models.
We further integrate MTBNs into a widely used PPL system, BLOG, and verify the effectiveness of the new inference algorithms through representative examples.
% derive provably correct inference algorithms for MTBN
%sampling algorithms for the generalized models and integrate them into the BLOG PPL. 
\end{abstract}

\section{Introduction}\label{sec:intro}
As originally defined by \citet{pearl88bayesNets}, Bayesian networks express joint distributions over finite sets of random variables as products of conditional distributions. 
%\simon{Add some applications for Bayesian network}
Probabilistic programming languages (PPLs)~\cite{koller97_ppl,milch05_ijcai,goodman08_church,wood-aistats-2014} apply the same idea to potentially infinite sets of variables with general dependency structures. Thanks to their expressive power, PPLs have been used to solve many real-world applications, including Captcha~\citep{le2017inference}, seismic monitoring~\citep{arora2013net}, 3D pose estimation~\citep{kulkarni2015picture}, generating design suggestions~\citep{ritchie2015generating}, concept learning~\citep{lake2015human}, and cognitive science applications~\citep{stuhlmuller2014reasoning}.

%A major drawback of existing PPLs is that they can only support discrete and continuous random variables but not their mixtures. \hide{: some PPLs restrict its syntax to only support purely continuous or discrete variables while other PPLs may produce wrong results for probabilistic programs with discrete-continuous mixtures.}
In practical applications, we often have to deal with a mixture of continuous and discrete random variables. 
%Random variables can be mixed in several ways. 
%\simon{
%The following sentences are drawn from~\cite{gao2017estimating}.
%We need to adapt to our examples.
%}
%First, one random variable can be discrete whereas the other is continuous. 
%For example, we want to measure the strength of relationship between children’s age and height, here age $X$ is discrete and height $Y$ is continuous. 
%Secondly, a single scalar random variable itself can be a mixture of discrete and continuous components. 
%For example, consider $X$ taking a zero-inflated-Gaussian distribution, which takes value 0 with probability p and is a Gaussian distribution with mean $\mu$ with probability $1-p$. 
%This distribution has both a discrete component as well as a component with
%density. 
%Finally, $X$ can be high dimensional vector, each of whose components may be discrete, continuous or mixed.
Existing PPLs support both discrete and continuous random variables, but not discrete-continuous mixtures, i.e., variables whose distributions combine discrete and continuous elements.
Such variables are fairly common in practical applications: sensors that have thresholded limits, e.g. thermometers, weighing scales, speedometers, pressure gauges; or a hybrid sensor that can report a either real value or an error condition.  
The occurrence of such variables has been noted in many other applications from a wide range of scientific domains~\citep{kharchenko2014bayesian,pierson2015zifa,gao2017estimating}.
%Combination of discrete and continuous distributions are ubiquitous in practical applications: 
%Such limits result in a discrete probability for the min and max values (capturing the integration of the values that are out of range), with a continuous distribution in between. 
%Such distributions cannot be handled by existing (even hybrid) approaches.\simon{why?}

Many PPLs have a restricted syntax that forces the expressed random variables to be either discrete or continuous, including WebPPL~\cite{dippl}, Edward~\cite{tran2016edward}, Figaro~\cite{pfeffer2009figaro} and Stan~\cite{carpenter2016stan}. Even for PPLs whose \emph{syntax} allows for mixtures of discrete and continuous variables, such as BLOG~\cite{milch05_ijcai}, Church~\cite{goodman2013principles}, Venture~\cite{mansinghka2014venture} and Anglican~\cite{wood2014new}, the underlying \emph{semantics} of these PPLs implicitly assumes the random variables are not mixtures. 
Moreover, the inference algorithms associated with the semantics inherit the same assumption and can produce incorrect results when discrete-continuous mixtures are used.

Consider the following GPA example: a two-variable Bayes net $\emph{Nationality} \rightarrow \emph{GPA}$ 
%\simon{better to draw a picture?} 
 where the nationality follows a binary distribution
{
  \vspace{-0.4em}
 $$
 P(\emph{Nationality} = \emph{USA})=P(\emph{Nationality}=\emph{India})=0.5
 \vspace{-0.4em}
 $$}and the conditional probabilities are discrete-continuous mixtures
 \hide{
\begin{align*}
&P\left(\mathrm{Nationality} = \textrm{USA}\right) \\
= &P\left(\textrm{Nationality} = \textrm{India}\right) =0.5
\end{align*}
}
  \vspace{-0.6em}
\begin{align*}
&\emph{GPA} | \emph{Nationality} = \emph{USA}\\
\sim &0.01 \cdot \indict\left\{\emph{GPA} = 4\right\} + 0.99 \cdot\unif(0,4),\\
&\emph{GPA} |  \emph{Nationality} = \emph{India}\\
\sim &0.01 \cdot \indict\left\{\emph{GPA} = 10\right\} + 0.99\cdot \unif(0,10).
\end{align*}
  \vspace{-2em}

This is a typical scenario in practice because many top students have perfect GPAs.
Now  suppose we observe a student with a GPA of 4.0. 
Where do they come from? 
If the student is Indian, the probability of any singleton set $\{g\}$ where $0 < g < 10$ is zero, as this range has a probability \emph{density}.
On the other hand if the student is American, the set $\{4\}$ has the probability $0.01$. Thus, by Bayes theorem, $P(\emph{Nationality} = \emph{USA}| \emph{GPA} = 4) =1$, which means the student \emph{must} be from the USA. 

However, if we run the default Bayesian inference algorithm for this problem in PPLs, e.g., the standard importance sampling algorithm~\cite{milch2005approximate}, a sample that picks India receives a density weight of $0.99/10.0 = 0.099$, whereas one that picks USA receives a discrete-mass weight of $0.01$.  
Since the algorithm does not distinguish probability density and mass, it will conclude that the student is very probably from India, which is far from the truth.

We can fix the GPA example by considering a density weight infinitely smaller than a discrete-mass weight~\citep{nitti2016probabilistic,tolpin2016design}.
However, the situation becomes more complicated when involving more than one evidence variable, e.g., GPAs over multiple semesters for students who may study in both countries.
Vector-valued variables also cause problems---does a point
mass in three dimensions count more or less than a point mass in two
dimensions?
These practical issues motivate the following two tasks:
%\vspace{-0.5em}
\begin{itemize}
\item Inherit all the existing properties of PPL semantics and extend it to handle random variables with mixed discrete and continuous distributions;
\item Design provably correct inference algorithms for the extended semantics.
\end{itemize}
%  \vspace{-0.5em}
In this paper, we carry out all these two tasks and implement the extended semantics as well as the new algorithms in a widely used PPL, Bayesian Logic (BLOG)~\cite{milch05_ijcai}.
% by proposing a general PPL semantics as well as provably correct inference algorithms.

%In this paper, we provide a general framework for the generalized PPL semantics and propose new algorithms based on this framework to handle mixture of discrete and continuous variables.
%Notably, our generalized framework can be also applied to other general measure spaces beyond discrete-continuous mixtures.

%\footnote{(1) Random variables with infinitely (even uncountably) many parents; (2) Establishment of conditional independencies implied by an infinite graph; and (3) Open-universe semantics in terms of the possible worlds in the vocabulary of the model.}

\subsection{Main Contributions}\label{sec:contrib}
\paragraph{Measure-Theoretical Bayesian Nets (MTBNs)}
Measure theory can be applied to handle discrete-continuous mixtures or even more abstract measures.
In this paper,  we define a generalization of Bayesian networks called \emph{measure-theoretic Bayesian networks (MTBNs)}
and prove that every MTBN represents a unique measure on the input space.
 We then show how MTBNs can provide a more general semantic foundation for PPLs.
 
More concretely, MTBNs support (1) random variables with infinitely (even uncountably) many parents, (2) random variables valued in \emph{arbitrary measure spaces} (with $\mathbb{R}^N$ as one case) distributed according to \emph{any measure} (including discrete, continuous and mixed), (3) establishment of conditional independencies implied by an infinite graph, and (4) open-universe semantics in terms of the possible worlds in the vocabulary of the model.

\paragraph{Inference Algorithms} 
We propose a provably correct inference algorithm, lexicographic likelihood weighting (LLW), for general MTBNs with discrete-continuous mixtures. In addition, we propose LPF, a particle-filtering variant of LLW for sequential Monte Carlo (SMC) inference on state-space models.
%Unlike discrete or continuous random variables, the existence of discrete-continuous mixtures often requires a user to incorporate specialized tricks for different models on top of classical PPL inference algorithms.
%Based on MTBN, we propose a general and provably correct inference algorithm, lexicographic likelihood weighting (LLW). In addition, we further adapt LLW into the sequential Monte Carlo (SMC) framework for state space models.
% , which can handle probabilistic programs with arbitrary mixtures
%\simon{some intuition of LLW, basic ideas?}
%Further, for state space models, we adapt LLW into the sequential Monte Carlo (SMC) framework to speed up the convergence.

\paragraph{Incorporating MTBNs into an existing PPL}
%\simon{maybe other word instead of instantiation}
%We incorporate MTBNs to a widely used PPL, Bayesian Logic (BLOG)~\citep{milch05_ijcai}.
%With simple modifications, we define the generalized BLOG language, 
We incorporate MTBNs into BLOG with simple modifications and then define the generalized BLOG language, 
\emph{measure-theoretic BLOG}, which formally supports arbitrary distributions, including discrete-continuous mixtures. 
We prove that every generalized BLOG model corresponds to a unique MTBN.
Thus, all the desired theoretical properties of MTBNs can be carried to measure-theoretic BLOG.
We also implement the LLW and LPF algorithms in the backend of measure-theoretic BLOG and use three representative examples to show their effectiveness.

%What about continuous variables with infinitely many point masses in a finite range?
%Existing semantics of PPLs are encumbered by several additional limitations:
%What about continuous random fields, which introduce uncountably many variables?
%What about variables that naturally require infinitely many parents, such as the time at which a Markov chain escapes a region?
%To cover such cases, we would ideally like a PPL to simultaneously support:
%% What about PPLs where two different samples may have different, unbounded numbers of variable values?

%\begin{enumerate}
%\vspace{-1em}
%  \setlength{\itemsep}
%	\item Random variables with infinitely (even uncountably) many parents,
%	\item Random variables valued in arbitrary measure spaces (with $\R^N$ as one case) distributed according to any measure (including discrete, continuous and mixed),
%	\item Establishment of conditional independencies implied by an infinite graph, and
%    \item Open-universe semantics in terms of the possible worlds in the vocabulary of the model.
%\vspace{-1em}
%\end{enumerate}

%Existing approaches do not handle most of these points (Sec.~\ref{sec:discuss}).
%In particular, since they typically draw upon Kolmogorov's existence theorem~\cite{durrett},
%they define the measure as a limit of a projective family over finite subsets of variables.
%As a consequence, they cannot assert conditional independencies involving infinitely many variables.
%In addition they rely on the assumption that each node has only finitely many parents to even define the projective family.

\subsection{Organization}\label{sec:org}
This paper is organized as follows. We first discuss related work in Section~\ref{sec:discuss}. 
In Section~\ref{sec:mtbn}, we formally define \emph{measure-theoretic Bayesian nets} and study their theoretical properties.
Section~\ref{sec:algo} describes the LLW and LPF inference algorithms for MTBNs with discrete-continuous mixtures and establishes their correctness.
In Section~\ref{sec:blog}, we introduce the measure-theoretic extension of BLOG and study its theoretical foundations for defining probabilistic models.
In Section~\ref{sec:expr}, we empirically validate the generalized BLOG system and the new inference algorithms on three representative examples.
%We conclude in Section~\ref{sec:con} and defer most technical details to appendix.

%In Section~\ref{sec:algo}, we propose a general inference algorithm, lexicographic likelihood weighting (LLW) for general MTBN and its variant, lexicographic particle filter (LPF), which is specialized for state space models.
%In this paper we present a measure theoretic formulation that provides all of the aforementioned desirable properties.
% We first define
%\emph{measure-theoretic Bayesian nets}~(MTBNs) in Sec.~\ref{sec:mtbn}, which can be used to provide semantics for any PPL, and then introduce the measure-theoretic extension of BLOG, an open-universe PPL~\cite{milch05_ijcai}, whose semantics inherently provide (4) above. We prove that every well-formed BLOG model corresponds to a unique MTBN  (Sec.~\ref{sec:blog}), and that every MTBN defines a unique probability measure with the properties (1-3) (Sec.~\ref{sec:proof}). In Sec.~\ref{sec:algo}, we present a provably correct sampling algorithm for MTBNs, lexicographic likelihood weighting (LLW), and further adapt LLW into the sequential Monte Carlo (SMC) framework for state space models. 
%\hide{We also implement the MTBN syntax and our proposed inference algorithms in an open-universe PPL, the BLOG language.} Finally, we  empirically validate the generalized BLOG system on a variety of models including the GPA example and many others (Sec.~\ref{sec:expr}). 
%We defer most technical details to appendix.

\section{Related Work}\label{sec:discuss}
The motivating GPA example has been also discussed as a special case under some other PPL systems~\cite{tolpin2016design,nitti2016probabilistic}. \citet{tolpin2016design} and \citet{nitti2016probabilistic} proposed different solutions specific to this example but did not address the general problems of representation and inference with random variables with mixtures of discrete and continuous distributions. In contrast, we present a general formulation with provably correct inference algorithms.
 \hide{In addition, \citet{shan2017exact} propose a program analysis approach to identify whether the input probabilistic program defines an ``ambiguous'' probability distribution such that classical MCMC algorithms produce wrong answers. This work is orthogonal to our focus.}

\hide{The closest related work to our framework is by \citet{milch06_thesis}, who utilize a
supportive numbering of random variables, implying that each random
variable has finitely many consistent parents. In addition, they only
handle random variables with countably infinite ranges.  }

Our approach builds upon the foundations of the BLOG probabilistic programming language~\cite{milch06_thesis}. We use a measure theoretic formulation to generalize the syntax and semantics of BLOG to random variables that may have infinitely many parents and mixed continuous and discrete distributions.
The BLP framework~\citet{kersting07_blp} unifies logic programming
with probability models, but requires each random variable to be
influenced by a finite set of random variables in order to define the
semantics. This amounts to requiring only finitely many ancestors of
each random variable. \citet{choi10_lifted} present an algorithm
for carrying out lifted inference over models with purely continuous
random variables.  They also require parfactors to be functions over
finitely many random variables, thus limiting the set of influencing
variables for each node to be finite. \citet{gutmann11_continuous_problog} also define densities
over finite dimensional vectors. In a relatively more general
formulation~\cite{gutmann11_magic} define the distribution of each
random variable using a definite clause, which corresponds to the
limitation that each random variable (either discrete or continuous)
has finitely many parents. Frameworks building on Markov networks also
have similar restrictions. \citet{Wang08_hybrid_mln}
only consider networks of finitely many random variables, which can
have either discrete or continuous distributions. \citet{singla07_markov} extend Markov logic to infinite
(non-hybrid) domains, provided that each random variable has only
finitely many influencing random variables. 
%These approaches do not address all the desired properties of an general PPL system.%  In particular, to our knowledge the problem of
% expressing a Bayes net while including individual random variables
% that have mixed (discrete and continuous) distributions has not been
% addressed. Furthermore, even when infinitely many random variables are
% allowed, each random variable is restricted to have only finitely many
% influencing random variables in existing approaches.

In contrast, our approach not only allows models with arbitrarily many
random variables with mixed discrete and continuous distributions, but
each random variable can also have arbitrarily many parents
as long as all ancestor chains are finite (but unbounded).
The presented work constitutes a rigorous framework for expressing
probability models with the broadest range of cardinalities
(uncountably infinite parent sets) and nature of random variables
(discrete, mixed, and even arbitrary measure spaces), with clear semantics in terms of
first-order possible worlds and the generalization of conditional
independences on such models. 

Lastly, there are also other works using measure-theoretic approaches to analyze the semantics properties of  probabilistic programs but with different emphases, such as the commutativity~\cite{staton2017commutative}, design choices for monad structures~\cite{ramsey2016all} and computing a disintegration~\cite{shan2017exact}.

% [Ackerman, Freer, Roy, 2011] discuss a significantly different problem: the computability of conditional distributions from *two-variable* joint distributions. Our paper is about mathematical representability of the reverse in a significantly more general setting.

\section{Measure-Theoretic Bayesian Networks}\label{sec:mtbn}
In this section, we introduce 
{\em measure-theoretic Bayesian networks (MTBNs)} and prove that an MTBN represents a unique measure with desired theoretical properties.
We assume familiarity with measure-theoretic approaches to probability theory.
Some background is included in Appx.~\ref{sec:back}.

We begin with some necessary definitions of graph theory.
\begin{definition}\label{defn:digraph}
A \keyword{digraph} $G$ is a pair $G=(V, E)$ of a set of vertices $V$, of any cardinality,
and a set of directed edges $E\subseteq V\times V$.
The notation $u\rightarrow v$ denotes $(u,v)\in E$, and $u\mapsto v$ denotes the existence of a path from $u$ to $v$ in $G$.
\end{definition}
\begin{definition}\label{def:root_vertex}
%, i.e., if there exists $n>0$ vertices $u_1,u_2,\dots\,u_n\in V$ where $u=u_1$ and $v=u_n$ such that $u_i\rightarrow u_{i+1}$.
A vertex $v\in V$ is a \keyword{root vertex} if there are no incoming edges to it,
i.e., there is no $u\in V$ such that $u\rightarrow v$.
Let $\pa(v) = \{u\in V : u\rightarrow v\}$ denote the set of parents of a vertex $v\in V$,
%$\de(v) = \{u\in V : v\mapsto u\}$ denote its set of descendants,
and 
	%$\nd(v) = V\setminus\de(v)$
	$\nd(v) = \{u\in V : \text{not } v\mapsto u\}$
denote its set of non-descendants.
\end{definition}
\begin{definition}\label{def:well_founded_digraph}
A \keyword{well-founded digraph} $(V,E)$ is one with no countably infinite ancestor chain $v_0 \leftarrow v_1 \leftarrow v_2 \leftarrow \dots$.
\end{definition}
%Equivalently, every invalid bigraph is one where every nonempty subset $U\subseteq V$ of vertices has at least one root vertex.
%
%\footnote{Infinite descendent sequences $v_0\rightarrow v_1\rightarrow v_2\rightarrow\dots$
%are fine.}
%
This is the natural generalization of a finite directed acyclic graph to the infinite case.
Now we are ready to give the key definition of this paper.

\begin{definition}\label{def:mtbn}
	A \keyword{measure-theoretic Bayesian network}
	$M  = (V, E, \{\X_v\}_{v \in V}, \{K_v\}_{v \in V})$ consists of
	(a) a well-founded digraph $(V, E)$ of any cardinality,
	(b) an arbitrary measurable space $\X_v$ for each $v\in V$, and
	(c) a probability kernel $K_v$ from $\prod_{u\in \pa(v)} \X_{u}$ to $\X_v$ for each $v\in V$.
\end{definition}
By definition, MTBNs allow us to define very general and abstract models with the following two major benefits:
%Specially, compared with the classical definition of Bayesian nets, 
%MTBN has the following two major benefits:
%\simon{Does the classical BN do not have these properties?}
\begin{enumerate}
	  \vspace{-0.5em}
\item We can define random variables with infinitely (even uncountably) many parents because MTBN is defined on a well-founded digraph.
\item We can define random variables in arbitrary measure spaces (with $\R^N$ as one case) distributed according to any measure (including discrete, continuous and mixed).
\end{enumerate}
  \vspace{-0.5em}

Next, we related MTBN to a probability measure.
Fix an MTBN $M = (V, E, \{\X_v\}_{v \in V}, \{K_v\}_{v \in V})$.
For $U\subseteq V$ let $\X_U = \prod_{u\in U}\X_u$ be the product measurable space over variables $u\in U$.
With this notation, $K_v$ is a kernel from $\X_{\pa(v)}$ to $\X_v$.
Whenever $W\subseteq U$ let $\pi^U_W\maps\X_U\to\X_W$ denote the projection map.
Let $\X_V$ be our base measurable space upon which we will consider different probability measures $\mu$.
Let $X_v$ for $v\in V$ denote both the underlying set of $\X_v$ and the random variable given by the projection $\pi^V_{\{v\}}$, 
and $X_U$ for $U\subseteq V$ the underlying space of $X_U$ and the random variable given by the projection $\pi^V_U$.

%% [[Define conditional independence?]]

\begin{definition}\label{def:mtbn-represents}
	An MTBN $M$ \keyword{represents} a measure $\mu$ on $\X_V$, if for all $v\in V$:
	\items{
		\item $X_v$ is conditionally independent of its non-descendants $X_{\nd(v)}$ given its parents $X_{\pa(v)}$.
		\item $K_v(X_{\pa(v)}, A) = \P_\mu[X_v\in A | X_{\pa(v)}]$ holds almost surely for any $A\in\X_v$,
		i.e., $K_v$ is a version of the conditional distribution of $X_v$ given its parents.  }
\end{definition}

Def.~\ref{def:mtbn-represents} captures the generalization of the local properties of Bayes networks -- conditional independence and conditional distributions defined by parent-child relationships. 
Here we assume the conditional probability exists and is unique.
This is a mild condition because this holds as long as the probability space is regular~\citep{kallenberg}.

\hide{
Then we show in Sec.~\ref{sec:proof} that an MTBN represents a unique measure:
}
The next theorem shows that MTBNs are well-defined.
\begin{theorem}\label{thm:represent}
	An MTBN $M$ represents a unique measure $\mu$ on $\X_V$.
\end{theorem}

\hide{
	
	Theorem~\ref{thm:represent} lays out the foundation of MTBN.
	Its proof requires a series of intermediate results.
	We first define a projective family of measures. }

The proof of theorem~\ref{thm:represent} requires several intermediate results and is presented in Appx.~\ref{sec:proof}. The proof proceeds by first defining a projective family of measures.
This gives a way to recursively construct our measure $\mu$.
We then define a notion of consistency such that every consistent projective family constructs a measure that $M$ represents.
Lastly, we give an explicit characterization of the unique consistent projective family, 
and thus of the unique measure $M$ represents.

%The full proof is in Appx.~\ref{sec:proof}.

\section{Generalized Inference Algorithms}\label{sec:algo}
We introduce the lexicographic likelihood weighting (LLW) algorithm for provably correct inference on MTBNs. 
We also present lexicographic particle filter (LPF) for state-space models by adapting  LLW for the sequential Monte Carlo (SMC) framework.

\hide{\reminder{do we need to add IRLW back for better consistency???}}

\subsection{Lexicographic likelihood weighting}
Suppose we have an MTBN with finitely many random variables $X_1,\dots,X_N$,
and that, without loss of generality, we observe real-valued random variables $X_1,\dots,X_M$ for $M<N$
as evidence. Suppose the distribution of $X_i$ given its parents $X_{\pa(i)}$
is a mixture between a density $f_i(x_i|x_{\pa(i)})$ with respect to the Lebesgue measure 
and a discrete distribution $F_i(x_i|x_{\pa(i)})$, i.e., for any $\epsilon>0$, we have
$
	P(X_i\in[x_i-\epsilon,x_i]|X_{\pa(i)})
	= \sum_{x\in [x_i-\epsilon,x_i]} F_i(x_i|X_{\pa(i)})
	+ \int_{x_i-\epsilon}^{x_i} f_i(x|X_{\pa(i)})\,dx.
$
This implies that $F_i(x_i|x_{\pa(i)})$ is nonzero for at most countably
many values $x_i$. If $F_i$ is nonzero for finitely many points, it can be
represented by a list of those points and their values.

%\keyword{Lexicographic Likelihood Weighting}
Lexicographic Likelihood Weighting (LLW) extends the classical likelihood weighting~\cite{milch2005approximate}
to this setting. It visits each node of the graph in topological order,
sampling those variables that are not observed,
and accumulating a weight for those that are observed. In particular,
at an evidence variable $X_i$ we update a tuple $(d,w)$ of the number of densities and a weight, initially $(0,1)$, by:
%\vspace{-0.7em}
\begin{equation}\label{eq:llw}
	(d,w) \gets		\begin{cases}
						(d, wF_i(x_i|x_{\pa(i)}))	&F_i(x_i|x_{\pa(i)})>0, \\
						(d+1, wf_i(x_i|x_{\pa(i)}))	&\text{otherwise.}						
					\end{cases}					
\end{equation}
%

%\label{sec:experiments}
%\begin{wrapfigure}{L}{3.5in}
%\includegraphics[scale=0.28]{coin_plot.pdf}
%\caption{\small Estimated posterior for the fake-coin
%  example. }\label{fig:coin_plot}
%\vspace{-10pt}
%\end{wrapfigure}
%\vspace{-0.7em}
Finally, having $K$ samples $x^{(1)},\dots,x^{(K)}$ by this process and accordingly a tuple $(d^{(i)}, w^{(i)})$ for each sample $x^{(i)}$, let $d^*=\min_{i: w^{(i)}\neq0} d^{(i)}$ and estimate $E[f(X)|X_{1:M}]$ by
%\vspace{-0.6em}
\begin{equation}\label{eq:llw-final}
	\frac{\sum_{\{i:d^{(i)}=d^*\}} w^{(i)}\,f(x^{(i)})}{\sum_{\{i:d^{(i)}=d^*\}} w^{(i)}}.
\end{equation}
The algorithm is summarised in Alg.~\ref{alg:llw}
The next theorem shows this procedure is consistent.
\begin{theorem}\label{thm:LLW_consistent}
	LLW is consistent: \eqref{eq:llw-final} converges almost surely to $\E[f(X)|X_{1:M}]$.	
\end{theorem}

% 	\begin{proof}	
% 		Due to space constraints we only sketch the proof where the evidence variables are leaves.
% 		Let $x$ be a sample produced by the algorithm with number of densities and weight $(d,w)$.
% 		With $I_n = \prod_{i=1,\dots,m} (\alpha_n(x_i)-2^{-n},\alpha_n(x_i)]$ a $2^{-n}$-cube around
% 		$x_{1:m}$ we have
% 		\[
% 			\lim_{n\to\infty} \frac{P(X_{1:m}\in I_\epsilon|X_{m+1:n}=x_{m+1:n})}{w\,2^{-dn}} = 1.
% 		\]
% 		Using $I_n$ as an approximation scheme as in the previous section, the numerator in the above
% 		limit is the weight used by IRLW. But given the above limit, using $w\,2^{-dn}$ as the weight
% 		will give the same result in the limit. But then if we have $N$ samples, in the limit of $n\to\infty$
% 		only those samples $x^{(i)}$ with minimal $d^{(i)}$ will contribute to the estimation, and up to normalization
% 		they'll contribute weight $w^{(i)}$ to the estimation.
% 	\end{proof}	

\begin{algorithm}[tb]
 \caption{Lexicographic Likelihood Weighting}\label{alg:llw}
\begin{algorithmic}
\REQUIRE densities $f$, masses $F$, evidences $E$, and $K$.
 \FOR{$i=1 \ldots K$}
  \STATE sample all the ancestors of $E$ from prior
  \STATE compute $(d^{(i)},w^{(i)})$ by Eq.~\eqref{eq:llw}
 \ENDFOR
 \STATE $d^\star \gets \min_{i:w^{(i)}\ne 0} d^{(i)}$
 \STATE \textbf{Return} $\left(\sum_{i:d^{(i)}=d^\star} w^{(i)} f(x^{(i)})\right) / \left( \sum_{i:d^{(i)}=d^\star} w^{(i)} \right)$
\end{algorithmic}
\end{algorithm}

In order to prove Theorem~\ref{thm:LLW_consistent}, the main technique we adopt is to use a more restricted algorithm, the Iterative Refinement Likelihood Weighting (IRLW) as a reference.
%%In order to prove theorem~\ref{thm:LLW_consistent}, we firstly introduced a more restricted algorithm, the Iterative Refinement Likelihood Weighting (IRLW) algorithm.

\subsubsection{Iterative refinement likelihood weighting}
Suppose we want to approximate 
	the posterior distribution of an $\X$-valued random variable $X$
	conditional on a $\Y$-valued random variable $Y$,
	for arbitrary measure spaces $\X$ and $\Y$.
In general, there is no notion of a probability density of $Y$ given $X$ for  weighing samples. 
If, however, we could make a discrete approximation $Y_t$ of $Y$
then we could weight samples by the probability $P[Y_t=y_t|X]$. If we increase the
accuracy of the approximation with the number of samples, this should
converge in the limit. We show this is possible, if we are careful about how we approximate:

%%% Approximation scheme to all us to use limits to implement conditional expectation

\begin{definition}\label{def:approx_scheme}
	An \keyword{approximation scheme} for a measurable space $\Y$ 
	consists of a measurable space $\A$ and 
		measurable approximation functions $\alpha_i\maps\Y\to\A$ for $i=1,2,\dots$ and
		$\alpha^j_i\maps\A\to\A$ for $i<j$
	such that
		$\alpha_j \circ \alpha^j_i = \alpha_i$
	and 
		$y$ can be measurably recovered from the subsequence $\alpha_t(y),\alpha_{t+1}(y),\dots$ for any $t>0$.
\end{definition}

When $Y$ is a real-valued variable we will use the approximation scheme
	$\approximate{y}{n} = 2^{-n}\lceil 2^n y \rceil$
where $\lceil r \rceil$ denotes the ceiling of $r$, i.e., the smallest integer no smaller than it. 
Observe in this case that 
%\[
	$P(\approximate{Y}{n} = \approximate{y}{n}) = P(\approximate{y}{n}-2^{-n}<Y\le\approximate{y}{n})$
%\]
which we can compute from the CDF of $Y$.

\begin{lemma}\label{lem:approximate}	
	If $X, Y$ are real-valued random variables with $\E |X|<\infty$,
	then $\lim_{i\to\infty} \E[X|\approximate{Y}{i}] = \E[X|Y]$.
\end{lemma}
	\begin{proof}
		Let $\F_i = \sigma(\approximate{Y}{i})$ be the sigma algebra generated by $\approximate{Y}{i}$.
		Whenever $i\le j$ we have $\alpha_i(Y) = (\alpha_j \circ \alpha^j_i)(Y)$ %% $\approximate{Y}{i} = \approximate{\approximate{Y}{j}}{i}$
		and so $\F_i \subseteq \F_j$.
		This means $\E[X|\approximate{Y}{i}] = \E[X|\F_i]$ is a martingale, 
		so we can use martingale convergence results.
		In particular, since $\E|X|<\infty$
		\vspace{-0.5em}
		$$
			\E[X|\F_i] \to \E[X|\F_\infty]	\quad\text{a.s.\ and in $L^1$},
				\vspace{-0.5em}
		$$
		where $\F_\infty = \bigcup_i \F_i$ is the sigma-algebra generated by $\{\approximate{Y}{i} : i\in\N\}$
		(see Theorem 7.23 in \cite{kallenberg}).

		$Y$ is a measurable function of the sequence $(\approximate{Y}{1},\dots)$,
		as $\lim_{i\to\infty} \approximate{Y}{i} = Y$, 	
		and so $\sigma(Y)\subseteq\F_\infty$.
		By definition the sequence is a measurable function of $Y$, and so $\F_\infty\subseteq\sigma(Y)$,
		and so $\E[X|\F_\infty] = \E[X|Y]$ giving our result.	
	\end{proof}

Iterative refinement likelihood weighting (IRLW) samples $x^{(1)},\dots,x^{(K)}$
from the prior and evaluates:

\begin{equation}\label{eq:measure-importance}
	\frac{\sum_{i=1}^K P(\alpha_n(Y)|X=x^{(i)}) f(x^{(i)})}{\sum_{i=1}^K P(\alpha_n(Y)|X=x^{(i)})}
\end{equation}

Using Lemma~\ref{lem:approximate}, \ref{lem:importance}, and~\ref{lem:likelihood-weighting}, we can show IRLW is consistent.
\begin{theorem}\label{thm:irlw-convergence}
	IRLW is consistent: \eqref{eq:measure-importance} converges almost surely to $\E[f(X)|Y]$.
\end{theorem}

\subsubsection{Proof of Theorem~\ref{thm:LLW_consistent}}
Now we are ready to prove Theorem~\ref{thm:LLW_consistent}.
\begin{proof}[Proof of Theorem~\ref{thm:LLW_consistent}]	
	We prove the theorem for evidence variables that are leaves
	It is straightforward to extend the proof when the evidence variables are non-leaf nodes.
	Let $x$ be a sample produced by the algorithm with number of densities and weight $(d,w)$.
	With $I_n = \prod_{i=1\ldots M} (\alpha_n(x_i)-2^{-n},\alpha_n(x_i)]$ a $2^{-n}$-cube around 
	$x_{1:M}$ we have
	\[
		\lim_{n\to\infty} \frac{P(X_{1:M}\in I_n|X_{M+1:N}=x_{M+1:N})}{w\,2^{-dn}} = 1.
	\]
	Using $I_n$ as an approximation scheme by Def.~\ref{def:approx_scheme}, the numerator in the above
	limit is the weight used by IRLW. But given the above limit, using $w\,2^{-dn}$ as the weight
	will give the same result in the limit. Then if we have $K$ samples, in the limit of $n\to\infty$
	only those samples $x^{(i)}$ with minimal $d^{(i)}$ will contribute to the estimation, and up to normalization
	they will contribute weight $w^{(i)}$ to the estimation.
\end{proof}

%% Eval pdf at points, predicate which determines if there's a point mass at x and if so its value.
%% Finite case works.

\subsection{Lexicographic particle filter}
%Likelihood weighting based algorithms suffer from the curse of dimensionality. One important class of models with high dimensionality are state-space models.
We now consider inference in a special class of high-dimensional models known as state-space models, and show how LLW can be adapted to avoid the curse of dimensionality when used with such models.
 A  state-space model (SSM) consists of  latent states $\left\{X_{t} \right\}_{0\le t\le T}$ and the observations $\left \{Y_t \right \}_{0\le t\le T}$ with a special dependency structure where $\pa(Y_t)=X_t$ and $\pa(X_t)=X_{t-1}$ for $0<t\le T$.

%Sequential Monte Carlo (SMC)~\cite{doucet2001introduction}, i.e., particle filer, is a widely adopted class of methods for inference on SSMs.
SMC methods~\cite{doucet2001introduction}, also knowns as particle filters, are a widely used class of methods for inference on SSMs. Given the observed variables $\{Y_t\}_{0\le t\le T}$, the posterior distribution $P(X_t |Y_{0:t})$ is approximated by a set of $K$ particles where each particle $x_t^{(k)}$ represents a sample of
$\{X_i\}_{0\le i\le t}$. Particles are propagated forward through the transition model
$P(X_t|X_{t-1})$ and resampled at each time step $t$ according to the weight of each particle, which is defined by the likelihood of observation $Y_t$. 

In the MTBN setting, the distribution of $Y_t$\footnote{There can be multiple variables observed. Here the notation $Y_t$  denotes $\{Y_{t,i}\}_i$ for conciseness.} given its parent $X_t$ can be a mixture of density $f_t(y_t|x_t)$ and a discrete distribution $F_t(y_t|x_t)$. Hence, the resampling step in a particle filter should be accordingly modified: following the idea from LLW, when computing the  weight of a particle, we enumerate all the observations $y_{t,i}$ at time step $t$ and again update a tuple $(d, w)$, initially (0,1), by

\begin{equation}\label{eq:lpf}
%\small{
	(d,w) \gets		\begin{cases}
						(d, wF_t(y_{t,i}|x_t))	&\text{ $F_t(y_{t,i}|x_t)>0$,} \\
						(d+1, wf_t(y_{t,i}|x_t))	&\text{otherwise.}						
					\end{cases}
%					}
\end{equation}

We discard all those particles with a non-minimum $d$ value and then perform the normal resampling step. We call this algorithm  lexicographical particle filter (LPF), which is summarized in Alg.~\ref{alg:lpf}.

\begin{algorithm}[tb]
 \caption{Lexicographic Particle Filter (LPF)}\label{alg:lpf}
\begin{algorithmic}
 \REQUIRE densities $f$, masses $F$, evidences $Y$, and $K$
% \KwResult{$Results$}
% $Results\gets \{\}$\;
\FOR{$t=0, \ldots, T$}
\FOR{$k = 0, \ldots, K$}
       \STATE $x_t^{(k)}\gets$ sample from transition
  \STATE compute $(d^{(k)},w^{(k)})$ by Eq.~\ref{eq:lpf}
\ENDFOR
  \STATE $d^\star \gets \min_{k:w^{(k)}\ne 0} d^{(k)}$
  \STATE $\forall k:d^{(k)}>d^\star$, $w^{(k)}\gets 0$
%append($Results$, $\left(w^{(k)}f(x_t^{(k)})\right)/\left(\sum_k w^{(k)}\right)$)\;
  \STATE Output $\left(w^{(k)}f(x_t^{(k)})\right)/\left(\sum_k w^{(k)}\right)$
  \STATE resample particles according to $w^{(k)}$
\ENDFOR
\end{algorithmic}
\end{algorithm}

The following theorem guarantees the correctness of LPF.
Its Proof easily follows the analysis for LLW and the classical proof of particle filtering based on importance sampling.
\begin{theorem}\label{thm:LPF_consistent}
	LPF is consistent: the outputs of Alg.~\ref{alg:lpf} converges almost surely to $\{E[f(X_t)|Y_{0:t}]\}_{0\le t\le T}$.	
\end{theorem}

%The correctness of LLW is the fundamental part for the LPF algorithm. Assuming the LLW, LPF can be easily derived by following the classical proof from importance sampling to particle filter. So now we focus on the proof for LLW algorithm.

\section{Generalized Probabilistic Programming Languages}\label{sec:blog}
In Section~\ref{sec:mtbn} and Section~\ref{sec:algo} we provided the theoretical foundation of MTBN and general inference algorithms.
%\section{Measure-theoretic BLOG}
%\label{sec:blog}
This section describes how to incorporate MTBN into a practical PPL. 
We focus on a widely used open-universe PPL, BLOG~\cite{milch06_thesis}. 
We define the generalized BLOG language, the \emph{measure-theoretic BLOG}, and prove that every well-formed measure-theoretic BLOG model corresponds to a unique MTBN. 
%Thanks to its open-universe semantics (property (4) in Sec.~1), measure-theoretic BLOG satisfies all the aforementioned properties. 
Note that our approach also applies to other PPLs\footnote{It has been shown that BLOG has  equivalent semantics to other PPLs~\cite{wubfit,mcallester2008random}.}.

%\hide{
We begin with a
brief description of the core syntax of BLOG, with particular emphasis
on (1) number statements, which are critical for expressing open-universe models\footnote{The specialized syntax in BLOG to express models with infinite number of variables.}, and (2) new syntax for expressing MTBNs, i.e., the \texttt{Mix}  distribution. Further description of BLOG's syntax can be found in
\citet{li2013blog}.
%}
\subsection{Syntax of measure-theoretic BLOG}

\begin{figure}[bt]
\begin{scriptsize}
\begin{verbatim}
1 Type Applicant, Country;
2 distinct Country NewZealand, India, USA;
3 #Applicant(Nationality = c) ~
4  if (c==USA) then Poisson(50)
5  else Poisson(5);
6 origin Country Nationality(Applicant);
7 random Real GPA(Applicant s) ~
8  if Nationality(s) == USA then
9      Mix({ TruncatedGauss(3, 1, 0, 4) -> 0.9998,
10          4 -> 0.0001, 0 -> 0.0001})
11 else Mix({ TruncatedGauss(5, 4, 0, 10) -> 0.989,
12           10 -> 0.009, 0 -> 0.002});
13 random Applicant David ~ 
14     UniformChoice({a for Applicant a});
15 obs GPA(David) = 4;
16 query Nationality(David) = USA;
\end{verbatim}
\end{scriptsize}
\vspace{-0.5em}
\caption{\small A BLOG code for the GPA
  example.}\label{fig:gpa_blog}
  \vspace{-0.7em}
\end{figure}

Fig.~\ref{fig:gpa_blog} shows a BLOG model with measure-theoretic extensions for a multi-student GPA example. Line 1 declares two \emph{types}, \emph{Applicant} and \emph{Country}. Line 2 defines 3 distinct countries with keyword \texttt{distinct}, New Zealand, India and USA. Lines 3 to 5 define a \emph{number statement}, which states that the number of
US applicants follows a Poisson distribution with a higher
mean than those from New Zealand or India.
Line 6 defines an \emph{origin function}, which maps the object being generated to the arguments that were
used in the number statement that was responsible for generating it. Here \emph{Nationality} maps applicants to their nationalities. Lines 7 and 13 define two random variables by keyword \texttt{random}. Lines 7 to 12 state that the GPA of an applicant is distributed as a mixture of
weighted discrete and continuous distributions. 
For US
applicants, the range of values $0<\emph{GPA} <4$ follows
a truncated Gaussian with bounds 0 and 4 (line 9).  The probability mass outside the range is
attributed to the corresponding bounds: $P(\emph{GPA}=0)= P(\emph{GPA}=4) =10^{-4}$ (line 10).  GPA distributions
for other countries are specified similarly.
Line 13 defines a random applicant $David$. Line 15 states that the David's GPA is observed to be 4 and we query in line 16 whether David is from USA.

\hide{

Details of the original BLOG syntax can be found in \cite{li2013blog}. Here we particularly emphasize on two syntax: (1) the most important existing syntax, the number statement, for the open-universe semantics; and (2) the newly introduced MTBN related syntax, the \texttt{Mix}  distribution.
}

%\subsubsection{Number statements} 
%\vspace{-10pt}
\paragraph{Number Statement (line 3 to 5)}
\begin{wrapfigure}{L}{1.6in}
\vspace{-5pt}
\begin{scriptsize}
\begin{align*}
\#\textrm{Type}_i&(g_1 = x_1, \ldots, g_k=x_k)\sim\\
    &\textrm{              }
    \textrm{if } \varphi_1(\bar{y}_1)\textrm{ then } 
               c_1(\bar{e}_1)\\
    &\textrm{              }\textrm{else if } \varphi_2(\bar{y}_2) \textrm{ then } c_2(\bar{e}_2)\\
    &\textrm{              }\ldots\\
    &\textrm{              }\textrm{else } c_m(\bar{e}_m)
;
\end{align*}
\end{scriptsize}
\vspace{-20pt}
\caption{\small Syntax of number statements}\label{fig:number}
%\vspace{-8pt}

\end{wrapfigure}

\hide{Each type can have multiple number
statements and each number statement can take other objects as
arguments (line 3 in Fig.~\ref{fig:gpa_blog}). Any type
that does not have any distinct objects must have at least one number
statement.} 

%\simon{What's the purpose of this paragraph?}
Fig.~\ref{fig:number} shows the syntax of a number
statement for $\emph{Type}_i$.
In this specification, $g_j$ are origin functions (discussed below);
$\bar{y}_j$ are tuples of arguments drawn from $\bar{x}=x_1, \ldots,
x_k$; $\varphi_j$ are first-order formulas with free variables
$\bar{y}_j$; $\bar{e}_j$ are tuples of expressions over a subset of
$x_1, \ldots, x_k$; and $c_j(\bar{e}_j)$ specify kernels $\kappa_j:
\Pi_{\set{\X_{\tau_e}: e \in \bar{e}_j}}\X_e\rightarrow \mathbb{N}$
where $\tau_e$ is the type of the expression $e$.

The arguments $\bar{x}$ provided in a number statement allow one to
utilize information about the rest of the model (and possibly other
generated objects) while describing the number of objects that should
be generated for each type. These assignments can be recovered using
the origin functions $g_j$, each of which is declared as:
\vspace{-0.3em}
$$
\mathtt{origin}~\textrm{Type}_j~~g_j(\textrm{Type}_i),
\vspace{-0.2em}
$$
where $Type_j$ is the type of the argument $x_j$ in the number
statement of $Type_i$ where $g_j$ was used. The value of the $j^{th}$
variable used in the number statement that generated $u$, an element
of the universe, is given by $g_j(u)$. Line 6 in Fig.~\ref{fig:gpa_blog} is an example of origin function.

\hide{
\mysssection{Dependency statements} A BLOG model may include
\texttt{random} functions, whose interpretations are given by
distributions (e.g., lines 7-12 in Fig.\,\ref{fig:blog_model_eg}), as
well as \texttt{fixed} functions, which have the same predefined
interpretation in all possible worlds. Examples of fixed functions
include deterministic operations such as addition and subtraction.
Dependency statements for random functions are similar to number
statements.}
% , except that the distribution can be over an arbitrary
% measurable space, rather than $\mathbb{N}$.

% can return an
% arbitrary . Thus, the general form of a dependency statement for a
% random function is:
% \begin{align*}
% \textrm{random Type}_f \quad f(\textrm{Type}_1~ x_1, \ldots, \textrm{Type}_k~ x_k)&\{\\
%     &\!\!\!\!\!\!\!\!\!\!\!\!\!\!\!\!\!\!\!\!\!\!\!\!\!\!\!\!\!\!\!\!\!\!\!\!\!\!\!\!\!
%     \textrm{if } \varphi_1(\bar{y}_1)\textrm{ then } \sim
%                c_1(\bar{e}_1)\\
%     &\!\!\!\!\!\!\!\!\!\!\!\!\!\!\!\!\!\!\!\!\!\!\!\!\!\!\!\!\!\!\!\!\!\!\!\!\!\!\!\!\!
%     \textrm{elseif } \varphi_2(\bar{y}_2) \textrm{ then } \sim c_2(\bar{e}_2)\\
%     \vdots\\
%     &\!\!\!\!\!\!\!\!\!\!\!\!\!\!\!\!\!\!\!\!\!\!\!\!\!\!\!\!\!\!\!\!\!\!\!\!\!\!\!\!\!
%     \textrm{else } \sim c_m(\bar{e}_m)
% \}
% \end{align*}
% The constraints on $\bar{y}_j$ and $c_j(\bar{e}_j)$ are as stated for
% number statements, with the target measure  $\mathbb{N}$ substituted
% by $\X_{Type_f}$.

%\subsubsection{Mixture Distributions}
%\vspace{-10pt}
\paragraph{Mixture Distribution (line 9 to 12)}
In measure-theoretic BLOG, we introduce a new distribution, the mixture distribution (e.g., lines 9-10
in Fig.~\ref{fig:gpa_blog}). 
A mixture distribution is specified
as: 
\vspace{-0.5em}
$$
\mathtt{Mix}(\{ c_1(\bar{e}_1) \rightarrow w_1(\bar{e}'),\ldots,
c_k(\bar{e}_k) \rightarrow w_k(\bar{e}')\}),
\vspace{-0.5em}
$$ where $c_i$ are arbitrary
distributions, and $w_i$'s are arbitrary real valued functions that
sum to 1 for every possible assignment to their arguments:
$\forall\bar{e}' \sum_{i} w_i(\bar{e}')=1$.  Note that in our implementation of measure-theoretical BLOG, we only allow a \texttt{Mix} distribution to express a mixture of densities and masses for simplifying the system design, although it still possible to express the same semantics without \texttt{Mix}\hide{\footnote{
Each use of a \texttt{Mix} can be translated into a regular dependency statement
contingent on a function $f_w(\bar{e}')$ over a finite prefix of
$\mathbb{N}$ with the categorical distribution $\set{1 \rightarrow
  w_1(\bar{e}'), 2\rightarrow w_2(\bar{e}'), \ldots, k \rightarrow
  w_k(\bar{e}')}$. The distribution can then be expressed using
conditions of the form \emph{if $f_w(\bar{e}')==i$ then $\sim
  c_i(\bar{e}_i)$}.}}.

\subsection{Semantics of measure-theoretic BLOG}
In this section we present the semantics of measure-theoretic BLOG and its theoretical properties.
%\simon{Correct me if I am wrong}
Every BLOG model implicitly defines a first-order vocabulary
consisting of the set of functions and types mentioned in the
model. BLOG's semantics are based on the standard, open-universe
semantics of first-order logic. We first define the set of all
possible elements that may be generated for a BLOG model.

\begin{definition}
\label{def:possible_elements}The set of \emph{possible
    elements} $\mc{U}_\mc{M}$ for a BLOG model $\mc{M}$ with types
  $\set{\tau_1, \ldots, \tau_k}$ is $\bigcup_{j\in
    \mathbb{N}}\set{\mc{U}_j}$, where \items{
\item $\mc{U}_0 = \tuple{U^0_1, \ldots, U^0_k}$,   $U^0_j=\{c_j:c_j$ is a \texttt{distinct} $\tau_i$ constant
       in  $\mc{M}\}$
 \item $\mc{U}_{i+1} = \tuple{U^{i+1}_1, \ldots, U^{i+1}_k}$, where
   $U^{i+1}_{m} = U^i_m \cup \{u_{\nu, \bar{u}, m}: \nu(\bar{x})$ is a
   number statement of type $\tau_m$, $\bar{u}$ is a tuple of elements
   of the type of $\bar{x}$ from $\mc{U}^i$, $m \in \mathbb{N}\}$
}
\end{definition}

Def.~\ref{def:possible_elements} allows us to define the set of
random variables corresponding to a BLOG model.
\begin{definition}
\label{def:brv}The set of \emph{basic random variables} for a BLOG model
  $\mc{M}$, $BRV(\mc{M})$, consists of:
\items{
\item for each number statement $\nu(\bar{x})$, a number variable
  $V_\nu[\bar{u}]$ over the standard measurable space $\mathbb{N}$,
  where $\bar{u}$ is of the type of $\bar{x}$.% :\\
\item for each function $f(\bar{x})$ and tuple $\bar{u}$ from
  $\mc{U}_\mc{M}$ of the type of $\bar{x}$, a \emph{function
    application variable} $V_f[\bar{u}]$ with the measurable space
  $\mc{X}_{V_f[\bar{u}]}=\mc{X}_{\tau_f}$, where $\mc{X}_{\tau_f}$ is
  the measurable space corresponding to $\tau_f$, the return type of
  $f$.  } 
\end{definition}

We now define the space of consistent assignments to random
variables. % Such assignments define unique possible worlds under
% open-universe semantics (~\cite{russell09_aiama}, Ch. 14.6).
\begin{definition}
\label{def:consistent_assignments}An instantiation $\sigma$ of the basic RVs defined by
a BLOG model $\mc{M}$ is \emph{consistent} if and only if:
\items{
	\item For every element $u_{\nu, \bar{v}, i}$ used in an
	assignment of the form $\sigma(V_f[\bar{u}])=w$ or
	$\sigma(V_\nu[\bar{u}])=m>0$, $\sigma(V_\nu[\bar{v}])\ge i$;
	\item For every fixed function symbol $f$ with the interpretation
	$\tilde{f}$, $\sigma(V_f[\bar{u}])= \tilde{f}(\bar{u})$; and
	\item For every element $u_{\nu, \bar{u}=\tuple{u_1, \ldots, u_m},
		i}$, generated by the number statement $\nu$, with origin functions
	$g_1, \ldots, g_m$, for every $g_j\in \set{g_1, \ldots, g_m}$,
	$\sigma(V_{g_j}[u_{\nu, \bar{u}, i}]) = u_j$. That is, origin
	functions give correct inverse maps.
	%\item For all number variables, $\sigma(V_\nu[\bar{u}])<\infty$.
}
\end{definition}

\begin{lemma}\label{lemma:unique_world} Every consistent assignment $\sigma$ to the basic RVs for
  $\mc{M}$ defines a unique possible world in the vocabulary of
  $\mc{M}$.
\end{lemma}
% \begin{proof} The possible world $\tuple{U^\sigma, I^\sigma}$ is
%   defined as follows.
%   $U^\sigma = \tuple{U^\sigma_1, \ldots, U^\sigma_k}$, where
%   $U_j^\sigma=\{c_j: c_j$ is a distinct constant of type $\tau_j$ in
%   $\mc{M}\}$ $\cup$ $\set{u_{\nu, \bar{u}, l} \in \mc{U}_\mc{M}: \nu$
%     is a number statement of type $\tau_j$ and $\sigma(V_\nu[\bar{u}])\ge l}$.

%   $I^\sigma$ is defined as follows. For each function symbol
%   $f(\bar{x})$ in $\mc{M}$, for each tuple $\bar{u}$ of the type of
%   $\bar{x}$ constructed using elements of $U^\sigma$,
%   $[f]^\sigma(\bar{u}) = \sigma(V_f[\bar{u}])$. The element
%   $\sigma(V_f[\bar{u}])$ is a member of $U^\sigma$ because of the last
%   clause in the definition of consistent assignments
%   (Def.\,\ref{def:consistent_assignments}) and the construction of
%   $U^\sigma$.
% \end{proof}

% \nb{ss}{Result on generating finite poss worlds with prob 1 based on
%   sub-martingales, if t permits}
The proof of Lemma~\ref{lemma:unique_world} is in Appx.~\ref{sec:proof_lemma_unique_world}. In the following definition, we use the notation $e[\bar{u}/\bar{x}]$
to denote a substitution of every occurrence of the variable $x_i$ with
$u_i$ in the expression $e$. For any BLOG model $\mc{M}$, let
$V(\mc{M})=BRV(\mc{M})$; for each $v\in V$, $\X_v$ is the measurable
space corresponding to $v$. Let $E(\mc{M})$ consist of the following
edges for every number statement or function application statement of
the form $s(\bar{x})$:

\items{
\item The edge $(V_g[\bar{w}], V_s[\bar{u}])$ if $g$ is a function
  symbol in $\mc{M}$ such that $g(\bar{y})$ appears in $s(\bar{x})$,
  and either $g(\bar{w})=g(\bar{y})[\bar{u}/\bar{x}]$ or an occurrence
  of $g(\bar{y})$ in $s(\bar{x})$ uses quantified variables
  $z_1,\ldots,z_n$, $\bar{u}'$ is a tuple of elements of the type of
  $\bar{z}$ and $g(\bar{w}) =
  g(\bar{y})[\bar{u}/\bar{x}][\bar{u}'/\bar{z}]$.
\item The edge
  $(V_{\nu}[\bar{v}], V_s[\bar{u}])$, for element $u_{\nu, \bar{v}, i}\in\bar{u}$.  }

Note that the first set of edges defined in $E(\mc{M})$ above may
include infinitely many parents for $V_s[\bar{u}]$.  Let the
dependency statement in the BLOG model $\mc{M}$ corresponding to a
number or function variable $V_s[\bar{f}]$ be $s$. Let $expr(s)$ be
the set of expressions used in $s$.  Each such statement then defines
in a straightforward manner, a kernel
$K_{s(\bar{u})}:\X_{expr(s(\bar{u}))} \rightarrow
\X_{V_s[\bar{u}]}$. In order ensure consistent assignments, we include
a special value $\emph{null} \in \X_\tau$ for each $\tau$ in $\mc{M}$,
and require that $K_{s(\bar{u})}(\sigma(pa(V_s[\bar{u}])),
\set{null}^c)=0$ whenever $\sigma$ violates the first condition of
consistent assignments (Def.\,\ref{def:consistent_assignments}). In
other words, all the local kernels ensure are \emph{locally
  consistent}: variables involving an object $u_{\nu, \bar{u}, i}$ get
a non-\emph{null} assignment only if the assignment to its number
statement represents the generation of at least $i$ objects
($\sigma(V_\nu(\bar{u}))\ge i$).  Each kernel of the form
$K_{s(\bar{u})}$ can be transformed into a kernel
$K_{pa(V_s[\bar{u}])}$ from its parent vertices (representing basic
random variables) by composing the kernels determining the truth value
of each expression $e\in expr(v)$ in terms of the basic random
variables, with the kernel $K^e_{V_s[\bar{u}]}$. Let $\kappa(\mc{M}) =
\set{K_{pa(V_s[\bar{u}])}: V_s[\bar{u}] \in BRV(\mc{M})}$.

\define{The MTBN $M$ for a BLOG model $\mc{M}$ is defined using
  $V=V(\mc{M}), E=E(\mc{M})$, the set of measurable spaces
  $\set{\X_{v}:v\in BRV(\mc{M})}$ and the kernels for each vertex
  given by $\kappa(\mc{M})$.}

By Thm.\,\ref{thm:represent}, we have the main result of this section, which provides the theoretical foundation for the generalized BLOG language:
\begin{theorem}\label{thm:blog}
  If the MTBN $\mc{M}$ for a BLOG model is a well-founded digraph, then $\mc{M}$
  represents a unique measure $\mu$ on $\X_{BRV(\mc{M})}$.
\end{theorem}
%This theorem provides the theoretical foundation of the generalized BLOG language.

% We can state an additional result on the assignments.

% \begin{theorem}
%   If the MTBN $M=(V,E, \set{\X_v}_{v\in V}, \set{K_v}_{v\in V})$ for a
%   BLOG model consists of countably many random variables, the measure
%   of all inconsistent assignments is zero.
% \end{theorem}
% \begin{proof}
%   We need to show that the measure of $S=\{\sigma: \exists
%   v,w,u\,\, v=v_f[\bar{u}]\in V, w=V_\nu[\bar{u}']\in pa(v), \sigma(w)<k \emph{ and }
%     u_{\nu,\bar{u}',k}\in \bar{u} \emph{ but } \sigma(v)\ne \emph{null} \}$ is zero. But
%     $S=\cup_{V_f[\bar{u}]\in V}\set{\sigma: \exists
%   w,u\,\, w=V_\nu[\bar{u}']\in pa(v), \sigma(w)<k \emph{ and }
%     u_{\nu,\bar{u}',k}\in \bar{u}  \emph{ but } \sigma(v)\ne \emph{null} }$; $\mu(S)\le
%   \sum_{V_f[\bar{u}]\in V} \mu(S'(V_f[\bar{u}]))$ where $S'(V_f[\bar{u}])=\set{\sigma: \exists
%   w,u\,\, w=V_\nu[\bar{u}']\in pa(v), \sigma(w)<k \emph{ and }
%     u_{\nu,\bar{u}',k}\in \bar{u}   \emph{ but } \sigma(v)\ne \emph{null}})$. Finally,
%   $\mu(S'((V_f[\bar{u}])))=0$ by the local consistency requirement of
%   the kernels $K_{pa(v)}$ for all $v\in V$.
% \end{proof}

\begin{figure}[tb]
	\begin{scriptsize}
		\begin{verbatim}
		1 fixed Real sigma = 1.0; // stddev of observation
		2 random Real FakeCoinDiff ~
		3  TruncatedGaussian(0.5, 1, 0.1, 1);
		4 random Bool hasFakeCoin ~ BooleanDistrib(0.5);
		5 random Real obsDiff ~ if hasFakeCoin 
		6   then Gaussian(FakeCoinDiff, sigma*sigma)
		7   else Mix({ 0 -> 1.0 });
		8 obs obsDiff = 0;
		9 query hasFakeCoin;
		\end{verbatim}
	\end{scriptsize}
	\vspace{-0.5em}
	\caption{\small BLOG code for the Scale example}\label{fig:coin_blog}
	  \vspace{-0.5em}
\end{figure}

\section{Experiment Results}\label{sec:expr}
\begin{figure*}[ht]
\centering
\subfigure[GPA model]{  \label{fig:gpa_plot}
   \includegraphics[width=0.3\textwidth]{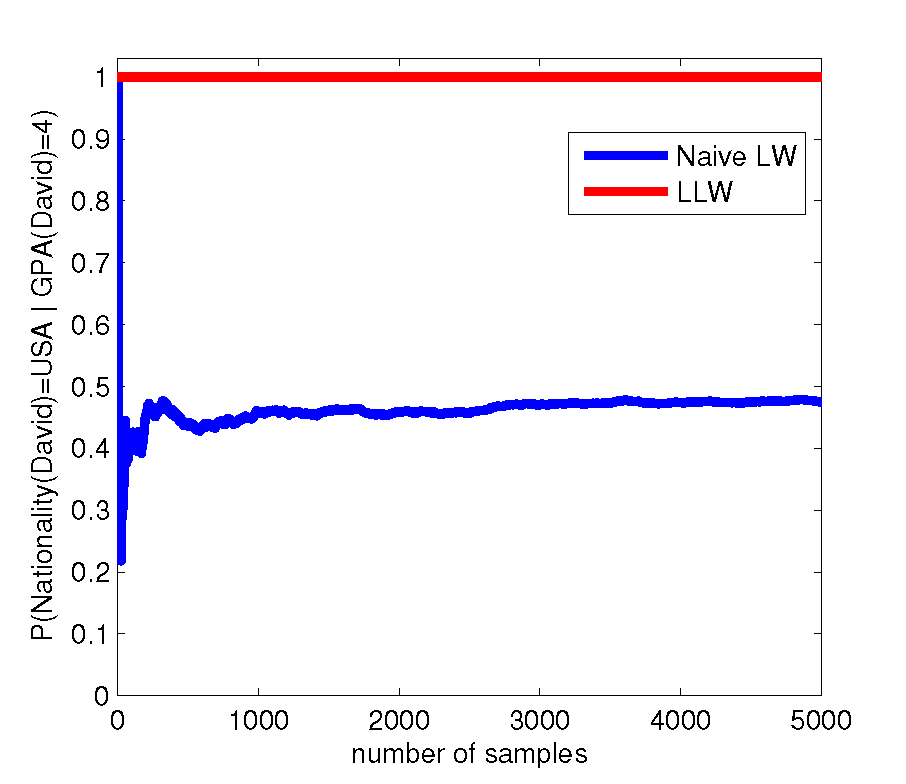}
}
\subfigure[Scale model]{  \label{fig:coin_plot}
   \includegraphics[width=0.3\textwidth]{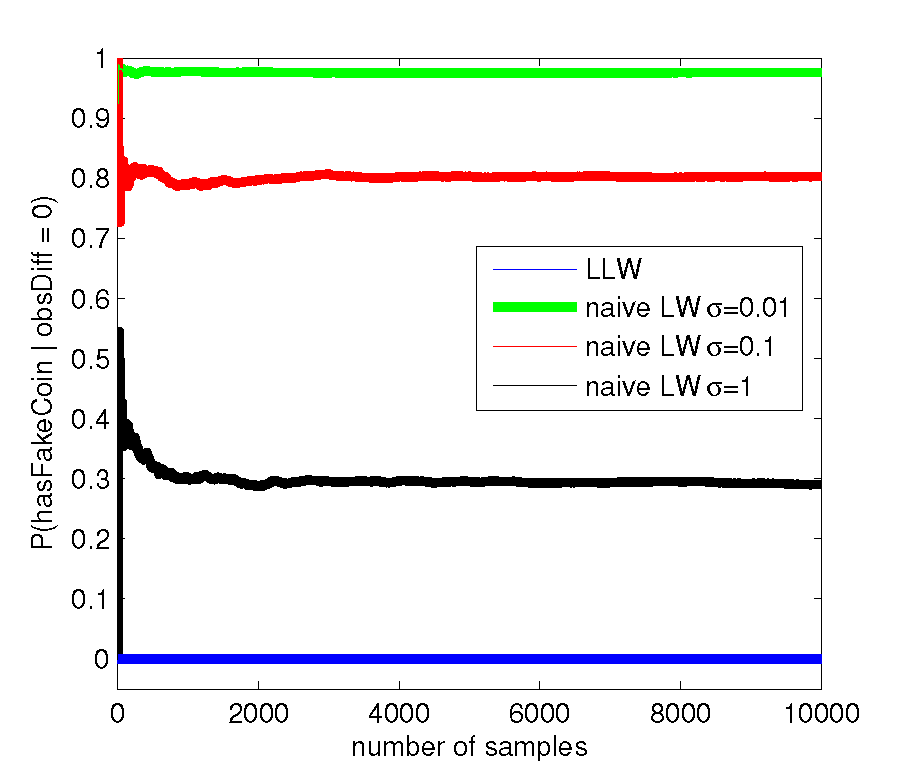}
}
\subfigure[Aircraft-Tracking model]{  \label{fig:track_plot}
   \includegraphics[width=0.3\textwidth]{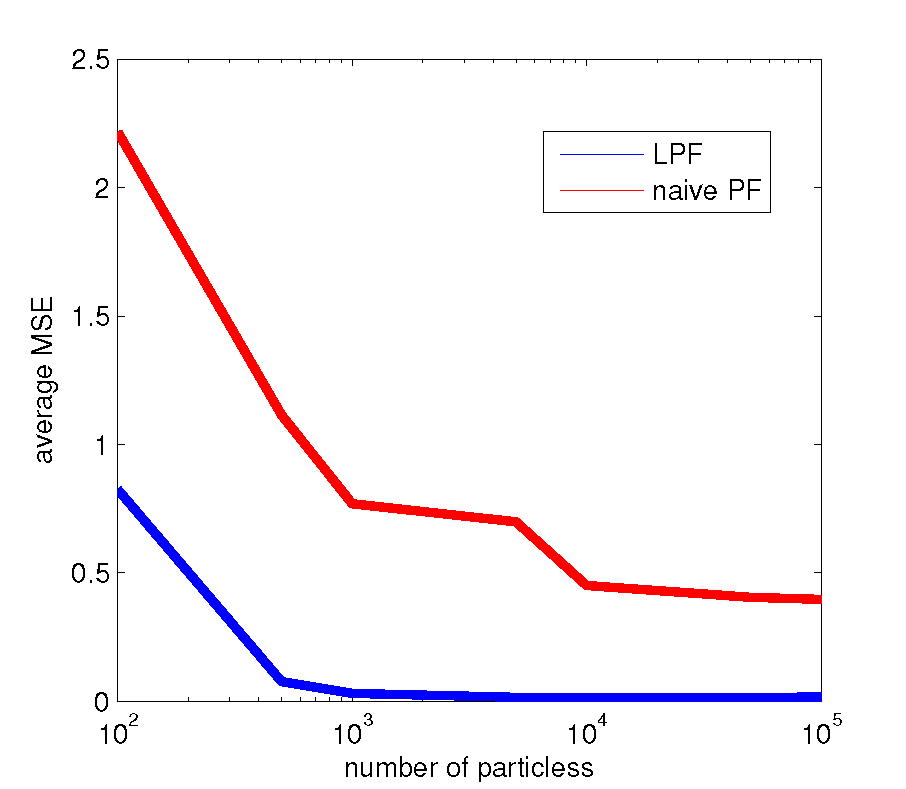}
}

\caption{\small Experiment results on (a) the GPA model, (b) the noisy scale model and (c) the aicraft-tracking model.}

\end{figure*}
%We implement the measure-theoretic extension and the two inference algorithms, LLW and LPF, based on the Swift compiler~\cite{yiIJCAI16} for the BLOG PPL, which contains some tiny syntax change to the original Java BLOG system~\cite{milch05_ijcai}.

We implemented the measure-theoretic extension of BLOG and evaluated our inference algorithms on three models where naive algorithms fail: (1) the GPA model (GPA); (2) the noisy scale model (Scale); and (3) a SSM, the aircraft tracking model (Aircraft-Tracking). The implementation is based on BLOG's C++ compiler~\cite{yiIJCAI16}.

\noindent\textbf{GPA model: } Fig.~\ref{fig:gpa_blog} presents the BLOG code for the GPA example as explained in Sec.~\ref{sec:blog}. Since the GPA of David is exactly 4, Bayes rule implies that David must be from USA. We evaluate LLW and the naive LW on this model in Fig~\ref{fig:gpa_plot}, where the naive LW converges to an incorrect posterior.

\noindent\textbf{Scale model: } In the noisy scale example (Fig.~\ref{fig:coin_blog}), we have an even number of coins and there might be a fake coin among them (Line 4). The fake coin will be slightly heavier than a normal coin (Line 2-3). We divide the coins into two halves and place them onto a noisy scale. When there is no fake coin, the scale always balances (Line 7). When there is a fake coin, the scale will noisily reflect the weight difference with standard deviation $\sigma$ (\texttt{sigma} in Line 6). Now we observe that the scale is  balanced (Line 8) and we would like to infer whether a fake coin exists. We again compare LLW against the naive LW with different choices of the $\sigma$ parameter in Fig.~\ref{fig:coin_plot}. Since the scale is precisely balanced, there must not be a fake coin. LLW always produces the correct answer but naive LW converges to different incorrect posteriors for different values of $\sigma$; as $\sigma$ increases, naive LW’s result approaches the true posterior.
%while for the naive LW, the result is incorrect and highly depends on the $\sigma$ parameter: as $\sigma$ increases, the output approaches the true probability.

\noindent\textbf{Aircraft-Tracking model: } Fig.~\ref{fig:track_blog} shows a simplified BLOG model for the aircraft tracking example. In this state-space model, we have $N=6$ radar points (Line 1) and a single aircraft to track. Both the radars and the aircraft are considered as points on a 2D plane. The prior of the aircraft movement is a Gaussian process (Line 3 to 6). Each radar $r$ has an effective range \texttt{radius(r)}: if the aircraft is within the range, the radar will noisily measure the distance from the aircraft to its own location (Line 13); if the aircraft is out of range, the radar will almost surely just output its radius (Line 10 to 11). Now we observe the measurements from all the radar points for $T$ time steps and we want to infer the location of the aircraft. 
With the measure-theoretic extension, a generalized BLOG program is more expressive for modeling  truncated sensors: if a radar outputs exactly its radius, we can surely infer that the aircraft must be out of the effective range of this radar. However, this information cannot be captured by the original BLOG language. To illustrate this case, we manually generated a synthesis dataset of $T=8$ time steps\footnote{The full BLOG programs with complete data are available at \url{https://goo.gl/f7qLwy}.} and evaluated LPF against the naive particle filter with different numbers of particles in Fig.~\ref{fig:track_plot}. We take the mean of the samples from all the particles as the predicted aircraft location. Since we know the ground truth, we measure the average mean square error between the true location and the prediction. LPF accurately predicts the true locations while the naive PF converges to the incorrect results.

\begin{figure}[bt]
\begin{scriptsize}
\begin{verbatim}
1  type t_radar; distinct t_radar R[6];
2  // model aircraft movement
3  random Real X(Timestep t) ~ if t == @0 
4    then Gaussian(2, 1) else Gaussian(X(prev(t)), 4);
5  random Real Y(Timestep t) ~ if t == @0 
6    then Gaussian(-1, 1) else Gaussian(Y(prev(t)), 4);
7  // observation model of radars
8  random Real obs_dist(Timestep t, t_radar r) ~
9    if dist(X(t),Y(t),r) > radius(r) then
10     mixed({radius(r)->0.999,
11    	TruncatedGauss(radius(r),0.01,0,radius(r))->0.001})
12   else
13     TruncatedGauss(dist(X(t),Y(t),r),0.01,0,radius(r));
14 // observation and query
15 obs obs_dist(@0, R[0]) = ...;
16 ... // evidence numbers omitted
17 query X(t) for Timestep t;
18 query Y(t) for Timestep t;
\end{verbatim}
\end{scriptsize}

\caption{\small BLOG code for the  Aircraft-Tracking example}\label{fig:track_blog}

\end{figure}
\section{Conclusion}\label{sec:con}
\hide{
In this paper, we proposed the measure-theoretic Bayesian network, a general framework to generalize existing PPL semantics to support random variables over arbitrary measure spaces. We then developed provably correct inference algorithms to handle discrete-continuous mixtures. We also incorporate MTBNs into a widely used PPL, BLOG, by a simple syntax extension and implement the new algorithms into the generalized BLOG language, which makes the system practical for a much larger domain of applications.
}

We presented a new formalization, measure-theoretic Bayesian networks, for generalizing the semantics of PPLs to include random variables with mixtures of discrete and continuous distributions. We developed provably correct inference algorithms for such random variables and incorporated MTBNs into a widely used PPL, BLOG. 
We believe that together with the
foundational inference algorithms, our proposed rigorous framework will
facilitate the development of powerful techniques for probabilistic
reasoning in practical applications from a much wider range of scientific areas.

\subsection*{Acknowledgment}
\small{
This work is supported by the DARPA PPAML program, contract FA8750-14-C-0011. Simon S. Du is funded by NSF grant IIS1563887,  AFRL grant FA8750-17-2-0212 and DARPA D17AP00001.}

\newpage

\bibliography{mtbn}
\bibliographystyle{icml2018}

\newpage

\twocolumn[
%\icmltitle{The Extended Semantics for Probabilistic Programming Languages\\ with Discrete-Continuous Mixtures} 
\icmltitle{Discrete-Continuous Mixtures in Probabilistic Programming:\\
	Generalized Semantics and Inference Algorithms\\
	Supplementary Materials}

\vskip 0.3in ]

\appendix

\section{Background on Measure-theoretical Probability Theory}\label{sec:back}

We assume familiarity with measure-theoretic approaches to probability
theory, but provide the fundamental definitions. The standard Borel $\sigma$-algebra is assumed in all the discussion. See
\cite{durrett} and \cite{kallenberg} for introduction and further
details.

A \keyword{measurable space} $(X, \X)$ (space, for short) is an underlying set
$X$ paired with a $\sigma$-algebra $\X\subseteq 2^X$ of measurable
subsets of $X$, i.e., a family of subsets containing the underlying
set $X$ which is closed under complements and countable unions.  We'll
denote the measurable space simply by $\X$ where no ambiguity
results. A function $f\maps\X\to\Y$ between measurable spaces is
measurable if measurable sets pullback to measurable sets:
$f^{-1}(B)\in\X$ for all $B\in\Y$. A \keyword{measure} $\mu$ on a
measurable space $\X$ is a function $\mu\maps\X\to[0,\infty]$ which
satisfies countable additivity: for any countable sequence $A_1,
A_2,\dots\in\X$ of disjoint measurable sets $\mu(\cup_{i} A_i) =
\sum_i \mu(A_i)$. $\P_\mu[S]$ denotes the probability
of a statement $S$ under the base measure $\mu$, and similarly for conditional probabilities.
A probability kernel is the measure-theoretic
generalization of a conditional distribution. It is commonly used to
construct measures over a product space, analogously to how
conditional distributions are used to define joint distributions in
the chain rule.

\begin{definition}
	A \keyword{probability kernel} $K$ from one measurable space $\X$ to another $\Y$
	is a function $K\maps X\times\Y\to[0,1]$ such that (a) for every
    $x\in X$, $K(x,\cdot)$ is a probability measure over $\Y$, and (b)
	for every $B\in\Y$, $K(\cdot,B)$ is a measurable function from $\X$ to $[0,1]$.
\end{definition}

Given an arbitrary index set $T$ and spaces $\X_t$ for each index $t
\in T$, the \keyword{product space} $\X = \prod_{t\in T} \X_t$ is the
space with underlying set $X = \prod_{t\in T} X_t$ the Cartesian
product of the underlying sets, adorned with the smallest
$\sigma$-algebra such that the projection functions
$\pi_t\maps\X\to\X_t$ are measurable.

%% [[More results here?]]

\hide{
\paragraph{First-Order Logic}
\hide{We define open-universe probabilistic models using the
BLOG language, whose syntax and semantics are based on typed
first-order logic.}  Given a set of types
$\mc{T}=\set{\tau_1, \ldots \tau_k}$, a first-order vocabulary is a
set of function symbols with their type signatures. Constant symbols
are represented as zero-ary function symbols and predicates as Boolean
functions.  Given a first-order vocabulary, a possible world
(``logical structure'' or a ``model structure'') is a tuple $\tuple{U,
  I}$ where the \emph{universe} $U=\tuple{U_1, \ldots, U_k}$ and each
$U_i$ is a set of elements of type $\tau_i\in\mc{T}$. The
\emph{interpretation} $I$ has, for each function symbol in the
vocabulary, a function of the corresponding type signature over $U$.}

\section{MTBNs Represent Unique Measures}\label{sec:proof}
%% [[Notation for projective family either created bottom up or top down.]]
%% [[Notation for projecting a measure down; marginalization.]]

%% Note here about why we need to do this for all orders:
%%    even if the conditional independencies are enough for any order,
%%    we need to guarantee that there's a single measure INDEPENDENT of the order of construction.

%\section{MTBNs represent unique measures}

We prove here \thmref{thm:represent}.
Its proof requires a series of intermediate results.
We first define a projective family of measures. This gives a way to recursively construct our measure $\mu$.
We define a notion of consistency such that every consistent projective family constructs a measure that $M$ represents.
We end by giving an explicit characterization of the unique consistent projective family, 
and thus of the unique measure $M$ represents.
The appendix contains additional technical material required in the proofs.

Intuitively, the main objective of this section is to show that an MTBN defines a unique measure that ``factorizes'' according to the network, as an extension to the corresponding result for Bayes Nets.

%====================================================
\subsection{Consistent projective family of measures}
%====================================================

%% Projective family

Let $K$ be a kernel from $\X\to\Y$ and $L$ a kernel from $\Y\to\Z$.
Their composition $K\circ L$ (note the ordering!) is a kernel from $\X$ to $\Z$ defined 
for $x\in X$,\hide{ and} $C\in\Z$ by:
\begin{equation}\label{eq:kernel-composition}
	(K\circ L)(x,C) = \int K(x,dy) \int L(y,dz)\,1_C(z).	
\end{equation}
To allow uniform notation, we will treat measurable functions and measures as special 
cases of kernels. A measurable function $f\maps X\to Y$ corresponds to the kernel 
$K_f$ from $\X$ to $\Y$ given by $K_f(x,B) = 1(f(x)\in B)$ for $x\in X$ and $B\in \Y$.
A measure $\mu$ on a space $\X$ is a kernel $K_\mu$ from $1$, the one element measure space, 
to $\X$ given by $K_\mu(\cdot,A) = \mu(A)$ for $A\in\X$.
Where this yields no confusion, we use $f$ and $\mu$ in place of $K_f$ and $K_\mu$.
\eqref{eq:kernel-composition} simplifies if the kernels are measures or functions.
Let
	$\mu$ be a measure on $\Y_1$,
	$K$ be a kernel from $\X_1$ to $\Y_1$,
	$f$ be a measurable function from $\X_2$ to $\X_1$, and
 	$g$ be a measurable function from $\Y_1$ to $\Y_2$.
Then $\mu \circ g$ is a measure on $\Y_2$ and $f\circ K \circ g$ is a kernel from $\X_2$ to $\Y_2$ with:
	$(\mu \circ g)(B) = \mu(g^{-1}(B))$, and
	$(f\circ K \circ g)(x,B) = K(f(x), g^{-1}(B))$.
%	\begin{align*}
%		(\mu \circ g)(B) 		&= \mu(g^{-1}(B)),\\
%		(f\circ K \circ g)(x,B)	&= K(f(x), g^{-1}(B)).
%	\end{align*}

Let $\Lambda$ denote the class of upwardly closed sets: subsets of $V$ containing all their elements' parents.

\begin{definition}\label{def:projective-family}
	A \keyword{projective family} of measures is a family $\{\mu_U : U\in\Lambda\}$ 
	consisting of a measure $\mu_U$ on $\X_U$ for every $U\in\Lambda$ such
	that whenever $W\subseteq U$ we have $\mu_{W} = \mu_{U} \circ \pi^U_W$,
	i.e., for all $A\in\X_W$, $\mu_W(A) = \mu_U((\pi^U_W)^{-1}(A))$.
\end{definition}	
	
Def.~\ref{def:projective-family} captures the measure-theoretic version of the probability of a subset of variables being equal to the marginals obtained while ``summing out'' the probabilities of the other variables in a joint distribution. 
\hide{We denote such a set of measures as a projective family.}

%% Consistent family

\begin{definition}\label{def:measure-kernel-product}
	Let $\mu$ be a measure on a measure space $\X$, and $K$ a kernel from $\X$ to a measure space $\Y$. 
	Then $\mu\otimes K$ is the measure on $\X\times\Y$ 
	defined for $B\in\X\otimes\Y$ by:
%	\[
		$(\mu\otimes K)(B) = \int \mu(dx) \int K(x,dy)\,1_B(x,y)$.
%	\]
\end{definition}

Def.~\ref{def:measure-kernel-product} defines the operation of composing a conditional probability with a prior on a parent, to obtain the corresponding joint distribution.

\begin{definition}\label{def:kernel}
	Let $K_w$ for $w\in W$ be kernels from $\X_U$ to $\X_{\{w\}}$.
	Denote by $\prod_{w\in W} K_w$ the kernel from $\X_U$ to $\X_W$
	defined for each $x_U\in \X_U$ by the infinite product of measures:
%	\[
		$\left(\prod_{w\in W} K_w\right)(x_U,\cdot) = \otimes_{w\in W} K_w(x_U, \cdot)$.
%	\]
\end{definition}

See \cite{kallenberg} 1.27 and 6.18 for definition and existence of infinite products of measures.
Def.~\ref{def:kernel} captures the kernel representation for taking the equivalent of products of conditional distributions of a set of variables with a common set $U$ of parents.

\begin{definition}\label{def:consistent-family}
	A projective family $\{\mu_U:U\in\Lambda\}$ is \keyword{consistent with $M$} 
	if for any $W,U\in\Lambda$ such that $W\subset U$ and $\pa(U)\subseteq W$, then:
%	\[
		$\mu_U = \mu_W \otimes \prod_{u\in U\setminus W} (\pi^W_{\pa(u)} \circ K_u)$.
%	\]
\end{definition}

Consistency in Def.~\ref{def:consistent-family} captures the global condition that we would like to see in a generalization of a Bayes network. Namely, the distribution of any set of parent-closed random variables should ``factorize'' according to the network

%% Consistent projective family iff represents

A projective family $\{\mu_U:U\in\Lambda\}$ is consistent with $M$ exactly when $M$ represents $\mu_V$:

\begin{lemma}\label{lem:mtbn-represents}
	Let $\mu$ be a measure on $\X_V$, and define the projective family $\{\mu_U:U\in\Lambda\}$  by $\mu_U = \mu\circ\pi^V_U$. 
	This projective family is consistent with $M$ iff $M$ represents $\mu$.
\end{lemma}
\begin{proof}
	First we'll relate consistency (Def. 8) 
	with conditional expectation and distribution properties of random variables.
	Take any $W,U\in\Lambda$ such that $W\subset U$ and $\pa(U)\subseteq W$
	%% in a prior iteration was: let $U\in\Lambda$ and $u\in U$ be such that $U\setminus\{u\}\in\Lambda$,
	%% here W = U\setminus\{u\}.
	%%
	and observe that the following are equivalent:
	\items{
		\item $\mu_U = \mu_W \otimes \prod_{u\in U\setminus W} (\pi^W_{\pa(u)} \circ K_u)$ 	
		\item $\prod_{u\in U\setminus W} (\pi^W_{\pa(u)} \circ K_u)$ 
		is a version of the conditional distribution of $X_{U\setminus W}$ given $X_W$,
		\item $K_u$ is a version of the conditional distribution of $X_u$ given $X_{\pa(u)}$ for all $u\in U\setminus W$,
		and $\{X_W, X_u : u\in U\setminus W\}$ are mutually independent conditional on $X_{\pa(U)}$.
	}
	%% [[Reasons by Kallenberg.]]
	
	The forward direction is straightforward. 
	For the converse we use the fact that conditional independence of families of random variables
	holds if it holds for all finite subsets, establishing that by chaining conditional independence 
	(see \cite{kallenberg} p109 and 6.8).
\end{proof}

Lemma~\ref{lem:mtbn-represents} shows that Def.~\ref{def:consistent-family} follows iff an MTBN represents the joint distribution -- in other words, it follows iff the local Markov property holds.

%===================================================
\subsection{There exists a unique consistent family}
%===================================================

Each vertex $v\in V$ is assigned the unique minimal ordinal 
$d(v)$ such that $d(u)<d(v)$ whenever $(u,v)\in E$ (see \cite{Jech:2003} for an introduction to ordinals).
For any $U\in\Lambda$ denote by $U^\alpha = \{u\in U : v(u)<\alpha\}$ the restriction of $U$ to vertices of depth less than $\alpha$.
Defining $D=\sup_{v\in V} (d(v)+1)$, the least strict upper bound on depth, we have that $U^D = U$ for all $U\in\Lambda$.
In the following, fix a limit ordinal $\lambda$.

\begin{definition}\label{def:projective-measures}
	$\{\nu_\alpha : \alpha<\lambda\}$ is a \keyword{projective sequence of measures} on $\X_{U_\alpha}$
	if whenever $\alpha<\beta<\lambda$ we have $\nu_\alpha = \nu_\beta \circ \pi^{U_\beta}_{U_\alpha}$.	
\end{definition}
\hide{10}

Def.~\ref{def:projective-measures} generalizes the notion of subset relationships and the marginalization operations that hold between supersets and subsets to the case of infinite dependency chains

\begin{definition}\label{def:measure-limit}
	The limit $\lim_{\alpha<\lambda}\nu_\alpha$
	of a projective sequence $\{\nu_\alpha : \alpha<\lambda\}$ of measures
	is the unique measure on $\X_U$ 
	such that $\nu_\alpha = (\lim_{\alpha<\lambda}\nu_\alpha) \circ \pi^{U}_{U_\alpha}$ for all $\alpha<\beta$.	
\end{definition}

% See \lemref{lem:measure-limit} in the appendix for existence and uniqueness of limits.

\begin{definition}\label{def:muU}
	Given any $U\in\Lambda$, inductively define a measure $\mu_U^\alpha$ on $\X_{U^\alpha}$ by
	\begin{align*}
		\mu_U^0	&= 1, \\
		\mu_U^{\alpha+1} &= \mu^{\alpha}_U \otimes\prod_{v\in U:d(v)=\alpha} (\pi^{U^\alpha}_{\pa(v)} \circ K_v), \\
		\mu_U^{\lambda} &= \lim_{\alpha<\lambda} \mu_U^\alpha \qquad\text{if $\lambda$ is a limit ordinal.}
	\end{align*}
	$\mu_U^\alpha$ stabilizes for $\alpha\ge D$ to define a measure on $\X_U$.
\end{definition}

The above definition is coherent as $\mu_U^{\alpha}$ can be inductively shown to be a projective sequence.
%using Lemmas~\ref{lem:product-project} and \ref{lem:measure-limit}.
Lemma~\ref{lem:projective} and~\ref{lem:consistent} allow us to show
in \thmref{lem:unique-consistent-family} that $\{\mu_U^D : U\in\Lambda\}$ is the unique consistent projective family of measures.
\hide{This combines with Lemma~\ref{lem:mtbn-represents} to prove \thmref{thm:represent}.}

\begin{lemma}\label{lem:projective}
	If $W\subseteq U$ for $W,U\in\Lambda$, then for all $\alpha$: 
		$\mu_W^\alpha = \mu_U^\alpha \circ \pi^{U^\alpha}_{W^\alpha}.$
\end{lemma}

Proof is in Appx.~\ref{sec:proof_lemma13}.

\begin{lemma}\label{lem:consistent}
	If $W\subset U$ where $W,U\in\Lambda$, and if $\pa(U)\subseteq W$, 
	then $W^\alpha\subset U^\alpha$, $\pa(U^\alpha)\subseteq W^\alpha$, and
		$\mu_U^\alpha = \mu_W^\alpha \otimes \prod_{u\in U^\alpha\setminus W^\alpha} (\pi^{W^\alpha}_{\pa(u)} \circ K_u).$
\end{lemma}

Proof is in Appx.~\ref{sec:proof_lemma14}.

Using the above, the following shows MTBNs satisfy the properties (1-3) mention in the beginning of Sec.~\ref{sec:contrib}:

\begin{theorem}\label{lem:unique-consistent-family}
	$\{\mu_U^D : U\in\Lambda\}$ is the unique projective family of measures consistent with $M$.
\end{theorem}

Proof is in Appx.~\ref{sec:proof_theorem15}.

Intuitively, by Lemma~\ref{lem:projective} and Lemma~\ref{lem:consistent}, we assert that consistency holds for any ordinal-bounded (prefix in terms of parent ordering) sub-network. Then the main result, Thm.~\ref{lem:unique-consistent-family}, follows by setting this bound appropriately. 
Finally Lemma~\ref{lem:mtbn-represents} and Theorem~\ref{lem:unique-consistent-family} lead to Theorem~\ref{thm:represent}.

Note that combining Thm.~\ref{thm:represent} and Thm.~\ref{thm:blog} lead to all the 4 desired properties mentioned in Sec.~\ref{sec:contrib}.
\section{Proof for Lemma~\ref{lem:projective}}\label{sec:proof_lemma13}
\begin{proof}
	Proof by induction. Trivially true for $\alpha=0$, so suppose this holds for $\alpha$, and consider $\alpha+1$. Then:
	\begin{align*}
		\mu^{\alpha+1}_W
		&= \mu^{\alpha}_W \otimes\prod_{v\in W:d(v)=\alpha} (\pi^{W^\alpha}_{\pa(v)} \circ K_v) \\
		&= \left(\mu_U^\alpha\circ\pi^{U^\alpha}_{W^\alpha}\right) \\
		&\qquad \otimes \left( 	\left(\prod_{v\in U:d(v)=\alpha} (\pi^{W^\alpha}_{\pa(v)}\circ K_v)\right) 
							\circ \pi^{U^{\alpha+1}\setminus U^{\alpha}}_{W^{\alpha+1}\setminus W^\alpha} \right) \\
		&= \left(\mu_U^\alpha \otimes\left(\pi^{U^\alpha}_{W^\alpha} \circ \prod_{v\in U:d(v)=\alpha} (\pi^{W^\alpha}_{\pa(v)} \circ K_v) \right)\right)\\
		&\qquad		\circ (\pi^{U^\alpha}_{W^\alpha} \times \pi^{U^{\alpha+1}\setminus U^{\alpha}}_{W^{\alpha+1}\setminus W^\alpha}) \\
%		&= \left(\mu_U^\alpha \otimes \prod_{v\in W:d(v)=\alpha} (K_v \circ \pi^{U^\alpha}_{\pa(v)})\right)
%						\circ (\pi^{U^\alpha}_{W^\alpha} \times 1_{?})^{-1} \\
		&= 				\left(\mu^{\alpha}_U \otimes\prod_{v\in U:d(v)=\alpha} (\pi^{U^\alpha}_{\pa(v)}\circ K_v)\right) 
							\circ \pi^{U^{\alpha+1}}_{W^{\alpha+1}} \\
		&= \mu_U^{\alpha+1} \circ \pi^{U^{\alpha+1}}_{W^{\alpha+1}}
	\end{align*}
	The first step by Def. 12,
	the second by inductive hypothesis and \lemref{lem:product-projection}
		as $\{v\in W:d(v)=\alpha\} = W^{\alpha+1}\setminus W^{\alpha}$
		and $\{v\in U:d(v)=\alpha\} = U^{\alpha+1}\setminus U^{\alpha}$,
	the third by \lemref{lem:product-exchange},
	the fourth by \lemref{lem:product-compose} since $\pi^{U^\alpha}_{\pa(v)} = \pi^{W^\alpha}_{\pa(v)} \circ \pi^{U^\alpha}_{W^\alpha}$
		and by elementary properties of projections,
	and the fifth by \defref{def:muU}.

	Finally, suppose $\lambda$ is a limit ordinal. We need to show:
	\[
		\lim_{\alpha<\lambda}\left( \mu_U^\alpha \circ \pi^{U^\alpha}_{W^\alpha} \right)
		= \left(\lim_{\alpha<\lambda} \mu_U^\alpha\right) \circ \pi^{U^\lambda}_{W^\lambda}.
	\]
	This follows from \lemref{lem:measure-limit} because for all $\alpha<\lambda$ we have:
	\begin{align*}
		\left(\left(\lim_{\alpha<\lambda} \mu_U^\alpha\right) \circ \pi^{U^\lambda}_{W^\lambda}\right) \circ \pi^{W^\lambda}_{W^\alpha}
		&= \left(\left(\lim_{\alpha<\lambda} \mu_U^\alpha\right) \circ \pi^{U^\lambda}_{U^\alpha}\right) \circ \pi^{U^\alpha}_{W^\alpha}\\
		&= \mu^\alpha_U \circ \pi^{U^\alpha}_{W^\alpha}
	\end{align*}
	The first by properties of projections, the second by \lemref{lem:measure-limit} characterizing limits.
\end{proof}

\section{Proof for Lemma~\ref{lem:consistent}}\label{sec:proof_lemma14}
\begin{proof}
	Trivial for $\alpha=0$, so suppose this holds for $\alpha$, and consider $\alpha+1$. Then:
\begin{align*}
& \mu_U^{\alpha+1} \\
&= \mu^{\alpha}_U \otimes\prod_{v\in U:d(v)=\alpha} (\pi^{U^\alpha}_{\pa(v)} \circ K_v) \\
&= \mu_W^\alpha \otimes \prod_{u\in U^\alpha\setminus W^\alpha} (\pi^{W^\alpha}_{\pa(u)}\circ K_u) 
			\otimes\prod_{v\in U:d(v)=\alpha} (\pi^{U^\alpha}_{\pa(v)}\circ K_v) \\
&= \mu_W^\alpha \otimes \prod_{u\in U^{\alpha+1}\setminus W^\alpha} (\pi^{W^\alpha}_{\pa(u)}\circ K_u) \\
&= \mu^{\alpha}_W \otimes\prod_{v\in W:d(v)=\alpha} (\pi^{W^\alpha}_{\pa(v)}\circ K_v)\\
& \qquad\qquad\qquad\qquad\qquad	\otimes \prod_{u\in U^{\alpha+1}\setminus W^{\alpha+1}} (\pi^{W^{\alpha+1}}_{\pa(u)} \circ K_u) \\
&= \mu^{\alpha+1}_W \otimes \prod_{u\in U^{\alpha+1}\setminus W^{\alpha+1}} (\pi^{W^{\alpha+1}}_{\pa(u)} \circ K_u),
\end{align*}
The first step by \defref{def:muU},
	the second by inductive hypothesis.
	The third by Lemmas~\ref{lem:otimes-prod} and~\ref{lem:prod-prod}
	 	since $U^{\alpha+1}\setminus W^\alpha = U^{\alpha}\setminus W^{\alpha} \cup \{v\in U: d(v)=\alpha\}$ where the union is disjoint,
		and as $\pa(v)\subseteq W^\alpha$ when $v\in U$ and $d(v)=\alpha$
		implies that $\pi^{U^\alpha}_{\pa(v)} = \pi^{U^\alpha}_{W^\alpha} \circ \pi^{W^\alpha}_{\pa(v)}$.
	The fourth by Lemmas~\ref{lem:otimes-prod} and~\ref{lem:prod-prod} 
	 	since $U^{\alpha+1}\setminus W^\alpha = U^{\alpha+1}\setminus W^{\alpha+1} \cup \{v\in W: d(v)=\alpha\}$ where the union is disjoint,
		and as $\pa(u)\subseteq W^\alpha$ when $u\in U^{\alpha+1}\setminus W^{\alpha+1}$
		implies that $\pi^{W^{\alpha+1}}_{\pa(v)} = \pi^{W^{\alpha+1}}_{W^\alpha} \circ \pi^{W^\alpha}_{\pa(v)}$.
	Finally, the fifth by \defref{def:muU}.

	Finally, suppose $\lambda$ is a limit ordinal. 
	The result will follow from the inductive hypothesis, \defref{def:muU},
	and as limits preserve products \lemref{lem:measure-kernel-product-limit}
	if we can show that 
	\[
		\lim_{\alpha<\lambda}\prod_{u\in U^\alpha\setminus W^\alpha} (\pi^{W^\alpha}_{\pa(u)} \circ K_u)
		= 
		\prod_{u\in U^\lambda\setminus W^\lambda} (\pi^{W^\lambda}_{\pa(u)} \circ K_u).
	\]
	First we must show the limit on the left is well-defined.
	Note that the kernel inside the limit maps from $\X_{W^\alpha}$ to $\X_{U^\alpha\setminus W^\alpha}$.
	As $W^\alpha$ and $U^\alpha\setminus W^\alpha$ are both increasing sets, 
	we verify projective sequence property by taking any $\beta>\alpha$ and observing that
	\begin{align*}
		&\pi^{W^\beta}_{W^\alpha} \circ \prod_{u\in U^\alpha\setminus W^\alpha} (\pi^{W^\alpha}_{\pa(u)} \circ K_u)\\
		&= \prod_{u\in U^\alpha\setminus W^\alpha} (\pi^{W^\beta}_{\pa(u)} \circ K_u) \\
		&= \left(\prod_{u\in U^\beta\setminus W^\beta} (\pi^{W^\beta}_{\pa(u)} \circ K_u)\right) 
				\circ \pi^{U^\beta\setminus W^\beta}_{U^\alpha\setminus W^\alpha}
	\end{align*} 
	the first step from \lemref{lem:product-compose} and properties of projections,
	and the second from \lemref{lem:product-projection}.
	
	Finally, we must show the expression on the right satisfies the properties characterizing the limit.
	However, observe this follows from our demonstration of the projective sequence property above by
	simply replacing $\beta$ with $\lambda$.
\end{proof}

\section{Proof for Theorem~\ref{lem:unique-consistent-family}}\label{sec:proof_theorem15}
\begin{proof}
	That this is a consistent projective family follows from Lemmas~\ref{lem:projective} and~\ref{lem:consistent}
	since $U^D = U$ for all $U\in\Lambda$.
	
	For uniqueness, let $\{\hat{\mu}_U : U\in\Lambda\}$ be a consistent projective family of measures,
	any fix any $U\in\Lambda$. We'll show inductively that $\hat{\mu}_{U^\alpha} = \mu_U^\alpha$,
	and thus with $\alpha=D$ that $\hat{\mu}_U = \mu_U$, giving our result.
	This is trivial for $\alpha=0$, so inductively suppose it holds for $\alpha$. But then:
	\begin{align*}
		\hat{\mu}_{U^{\alpha+1}}
		&= \hat{\mu}_{U^{\alpha}} \otimes \prod_{u\in U^{\alpha+1}\setminus U^{\alpha}} (\pi^{U^{\alpha}}_{\pa(u)} \circ K_u) \\
		&= \mu_U^\alpha \otimes \prod_{u\in U^{\alpha+1}\setminus U^{\alpha}} (\pi^{U^{\alpha}}_{\pa(u)} \circ K_u).
%		&= \mu_U^{\alpha+1},
	\end{align*}
	The first step by consistency of $\{\hat{\mu}_U\}$ (\defref{def:consistent-family}) 
		since $U^\alpha\subseteq U^{\alpha+1}$ and $\pa(U^{\alpha+1})\subseteq U^{\alpha}$,
	the second by inductive hypothesis,
	and the third by \defref{def:muU}.
	
	Let $\alpha$ be a limit ordinal.
	Since $\{\hat{\mu}_{U^\alpha}\}$ is a projective family and $U^\alpha = \bigcup_{\beta<\alpha} U^\beta$,
	by \lemref{lem:measure-limit} $\hat{\mu}_{U^\alpha} = \lim_{\beta<\alpha} \hat{\mu}_{U^\beta}$.
	By definition $\mu_U^\alpha = \lim_{\beta<\alpha} \mu_U^\beta$.
	Then since $\mu_U^\beta = \hat{\mu}_{U^\beta}$ for $\beta<\alpha$ inductively,
	$\mu_U^\alpha = \hat{\mu}_{U^\alpha}$ as the limit of this sequence is unique.
\end{proof}

\section{Proof of Lemma~\ref{lemma:unique_world}}\label{sec:proof_lemma_unique_world}
\begin{proof} The possible world $\tuple{U^\sigma, I^\sigma}$ is
  defined as follows.
  $U^\sigma = \tuple{U^\sigma_1, \ldots, U^\sigma_k}$, where
  $U_j^\sigma=\{c_j: c_j$ is a distinct constant of type $\tau_j$ in
  $\mc{M}\}$ $\cup$ $\set{u_{\nu, \bar{u}, l} \in \mc{U}_\mc{M}: \nu$
    is a number statement of type $\tau_j$ and $\sigma(V_\nu[\bar{u}])\ge l}$.

  $I^\sigma$ is defined as follows. For each function symbol
  $f(\bar{x})$ in $\mc{M}$, for each tuple $\bar{u}$ of the type of
  $\bar{x}$ constructed using elements of $U^\sigma$,
  $[f]^\sigma(\bar{u}) = \sigma(V_f[\bar{u}])$. The element
  $\sigma(V_f[\bar{u}])$ is a member of $U^\sigma$ because of the last
  clause in the definition of consistent assignments
  (Def.\,\ref{def:consistent_assignments}) and the construction of
  $U^\sigma$.
\end{proof}

\section{Additional Technical Details}

For reasons of space, we present the following without their (straightforward) proofs.

\begin{lemma}\label{lem:product-project}
If $\mu$ is a measure on $\X$, and is $K$ a kernel from $\X$ to $\Y$, then
	$(\mu\otimes K) \circ \pi^{\X\times\Y}_{\X} = \mu.$
\end{lemma}

\begin{lemma}\label{lem:measure-limit}
	A projective sequence of measures has a unique limit.
\end{lemma}
%\begin{proof}
%	A measure on a product space $\X_U$ is uniquely determined by its measure on finite cylinder sets, 
%	i.e., those of the form $A\times \X_{U\setminus F}$ for $A\in\X_F$ and $F\subseteq U$ finite (see Kallenberg 3.2).
% 	Define $\nu A\times \X_{U\setminus F} = \nu_\alpha A\times \X_{U_\alpha\setminus F}$ for
%	any $\alpha$ with $F\subseteq U_\alpha$ and observe consistency follows from $\{\nu_\alpha\}$ being projective.
%\end{proof}

Fix an ordinal $\lambda$,
and suppose $\{U_\alpha\subseteq V : \alpha<\lambda\}$ is an increasing sequence of subsets of $V$,
i.e., such that if $\alpha<\beta<\lambda$ then $U_\alpha\subseteq U_\beta$.
Define $U=\bigcup_{\alpha<\lambda}U_\alpha$.
Let $\{W_\alpha\subseteq V : \alpha<\lambda\}$ and $W$ be another such sequence, supposing $U$ and $W$ are disjoint.

\begin{definition}\label{def:projective-kernels}
	$\{K_\alpha : \alpha<\lambda\}$ is a \keyword{projective sequence of kernels} from $\X_{U_\alpha}$ to $\X_{W_\alpha}$
	if whenever $\alpha<\beta<\lambda$ we have
		$\pi^{U_\beta}_{U_\alpha} \circ K_\alpha = K_\beta \circ \pi^{W_\beta}_{W_\alpha}.$
\end{definition}

\begin{definition}\label{def:kernel-limit}
	The limit $\lim_{\alpha<\beta} K_\alpha$
	of a projective sequence $\{K_\alpha : \alpha<\lambda\}$ of kernels
	is the unique kernel from $\X_U$ to $\X_W$
	such that for all $\alpha<\lambda$
%	\[
		$\pi^{U}_{U_\alpha} \circ K_\alpha = \left( \lim_{\alpha<\beta} K_\alpha \right) \circ \pi^{W}_{W_\alpha}.$
%	\]
\end{definition}

\begin{lemma}\label{lem:kernel-limit}
	A projective sequence of kernels has a unique limit.
\end{lemma}

\begin{lemma}\label{lem:product-exchange}
Let $\X_1,\Y_1,\X_2,\Y_2$ be measurable spaces,
	$\mu$ be a measure on $\X_1$,
	$K$ a kernel from $\X_2$ to $\Y_1$,
	$f\maps\X_1\to\X_2$ a measurable function,
	and $g\maps\Y_1\to\Y_2$ a measurable function.
Then:
%\[
	$(\mu \otimes (f \circ K)) \circ (f \times g) = (\mu \circ f) \otimes (K \circ g)$
%\]
where $f \times g$ is the measurable function mapping $(x,y)$ to $((f(x), g(y))$.
\end{lemma}

\begin{lemma}\label{lem:measure-kernel-product-limit}
	Let $\nu_\alpha$ and $K_\alpha$ be as in Lemmas~\ref{lem:measure-limit} and~\ref{lem:kernel-limit}.
	Then
%	\[
$		\lim_{\alpha<\lambda} (\nu_\alpha \otimes K_\alpha)
		= (\lim_{\alpha<\lambda} \nu_\alpha) \otimes (\lim_{\alpha<\lambda} K_\alpha).$
%	\]
\end{lemma}

\begin{lemma}\label{lem:otimes-prod}
	$\mu$ measure on $\X$,
	$K_1$ a kernel from $\X$ to $\Y_1$,
	$K_2$ a kernel from $\X$ to $\Y_2$, 
%	\[
		$\mu \otimes K_1 \otimes (\pi^{\X\times\Y_1}_{\X} \circ K_2)
		= \mu \otimes \prod_{i=1,2} K_i$.
%	\]
	where by abuse of notation $\pi^{\X\times\Y_1}_{\X}$ denotes the projection from $\X\times\Y_1$ to $\X$.
\end{lemma}

\begin{lemma}\label{lem:prod-prod}
	If $K_{i,j}$ are kernels from $\X$ to $\Y_{i,j}$ then
%	\[
		$\prod_{i} \prod_j K_{i,j} = \prod_{i,j} K_{i,j}$.
%	\]
\end{lemma}
%\begin{proof}
%Follows from \defref{def:kernel} and basic properties of the independent product of measures.
%\end{proof}

\begin{lemma}\label{lem:product-compose}
	If $f\maps\X'\to\X$ and $K_i$ are kernels from $\X$ to $\Y_i$ then
%	\[
		$f\circ\prod_i K_i = \prod_i f\circ K_i$.
%	\]	
\end{lemma}

\begin{lemma}\label{lem:product-projection}
	If $K_v$ for $v\in U$ are kernels from $\X$ to $\X_v$, and $W\subseteq U$ then
%	\[
		$\left(\prod_{v\in U} K_v\right) \circ \pi^U_W = \prod_{v\in W} K_v$.
%	\]	
\end{lemma}

%In the limit, this converges almost surely to $\E[f(X)|Y]$ (Theorem~\ref{thm:irlw-convergence}).
\begin{lemma}\label{lem:importance}
	Let $(X,\X)$ be a measurable space, $X,X_1,X_2,\dots$ an iid random sequence on $\X$,
	and $w(x)$ be non-negative real-valued function of $(X,\X)$.
%	and $w(x), w_1(x), w_2(x)$ be non-negative real-valued function of $(X,\X)$
%	such that $w(x)/\E w(x) = \lim_{n\to\infty} (w_n(x)/\E w_n(X))$.
	Then
$		\frac{\sum_{i=1}^n w(X_i) f(X_i) }{\sum_{i=1}^n w(X_i)}
	\stackrel{\rm a.s.}{\to}
		\frac{\E w(X) f(X)}{\E w(X)}$.
%	\[
%		\frac{\sum_{i=1}^n w(X_i) f(X_i) }{\sum_{i=1}^n w(X_i)}
%	\stackrel{\rm a.s.}{\to}
%		\frac{\E w(X) f(X)}{\E w(X)}.
%	\]
\end{lemma}

\begin{lemma}\label{lem:likelihood-weighting}
 	For any measurable set $E$ and measurable function $f(x)$: $\frac{\E P(E|X) f(X)}{\E P(E|X)} = \E[f(X)|E].$
%	\[
%		\frac{\E P(E|X) f(X)}{\E P(E|X)} = \E[f(X)|E].
%	\]
\end{lemma}

\end{document}